\documentclass[11pt]{article}

\usepackage{times}
\usepackage{graphicx} % more modern
\usepackage{subfigure} 

\usepackage{amsfonts}
\usepackage{amsmath}
\usepackage{amssymb}
\usepackage{amsthm}
\usepackage{mathtools}
\usepackage{bbm}
\usepackage[toc,page]{appendix}

\usepackage{hhline}

\usepackage[round,authoryear]{natbib}

\usepackage{algorithm}
\usepackage{algorithmic}
\usepackage{hyperref}

\newtheorem{cor}{Corollary}
\newtheorem{defn}{Definition}

\newtheorem{lemma}{Lemma}
\newtheorem{assumption}{Assumption}
\newtheorem{thm}{Theorem}

\newcommand\numberthis{\addtocounter{equation}{1}\tag{\theequation}}

\DeclarePairedDelimiter{\ceil}{\lceil}{\rceil}

\DeclarePairedDelimiter\abs{\lvert}{\rvert}%
\DeclarePairedDelimiter\norm{\lVert}{\rVert}%

\usepackage{array}
\newcolumntype{L}[1]{>{\raggedright\let\newline\\\arraybackslash\hspace{0pt}}m{#1}}
\newcolumntype{C}[1]{>{\centering\let\newline\\\arraybackslash\hspace{0pt}}m{#1}}
\newcolumntype{R}[1]{>{\raggedleft\let\newline\\\arraybackslash\hspace{0pt}}m{#1}}

\newcommand*{\tabbox}[2][t]{%
    \vspace{0pt}\parbox[#1][2.5\baselineskip]{1cm}{\centering #2}}

\newcommand{\cellCenter}[1]{\multicolumn{1}{|c|}{#1} }
\newcommand{\eqdef}{\overset{\mathrm{def}}{=\joinrel=}}

\usepackage{chngcntr}
\counterwithout{figure}{section}
\counterwithout{figure}{subsection}
\counterwithout{figure}{subsubsection}

\begin{document}

\title{Stabilized Sparse Online Learning for Sparse Data}

\author{
  Yuting Ma\\
  \texttt{yma@stat.columbia.edu}\\
  Department of Statistics\\
  Columbia University\\
  New York, NY 10027
  \and
 	Tian Zheng\\
  \texttt{tzheng@stat.columbia.edu} \\
  Department of Statistics\\
  Columbia University\\
  New York, NY 10027
}
\date{}

\maketitle

\begin{abstract}
Stochastic gradient descent (SGD) is commonly used for optimization in large-scale machine learning problems. \cite{langford2009sparse} introduce a sparse online learning method to induce sparsity via truncated gradient. With high-dimensional sparse data, however, this method suffers from slow convergence and high variance due to heterogeneity in feature sparsity. To mitigate this issue, we introduce a stabilized truncated stochastic gradient descent algorithm. We employ a soft-thresholding scheme on the weight vector where the imposed shrinkage is adaptive to the amount of information available in each feature. The variability in the resulted sparse weight vector is further controlled by stability selection integrated with the informative truncation. To facilitate better convergence, we adopt an annealing strategy on the truncation rate, which leads to a balanced trade-off between exploration and exploitation in learning a sparse weight vector. Numerical experiments show that our algorithm compares favorably with the original truncated gradient SGD in terms of prediction accuracy, achieving both better sparsity and stability. 
\end{abstract}

%the high variability of approximated gradient using sparse data in a stochastic setting. To mitigate this issue, we introduce a stabilized sparse stochastic gradient descent algorithm with multi-thread SGD updates executed in parallel. We employ a soft-regularization scheme on the weight vector with adaptive gravity to shrink the weight of irrelevant features. The sparsity is further controlled by stability selection on nonzero features to reduce the variance. To facilitate better convergence, we adopt an annealing strategy on the rejection rate in obtaining an adaptive gravity, which leads to a balanced trade-off between exploration and exploitation in learning a sparse weight vector. 
%a resulting insensitivity to irrelevant attributes.


\section{Introduction}\label{sec_intro}

Modern datasets pose many challenges for existing learning algorithms due to their unprecedented large scales in both sample sizes and input dimensions. It demands both efficient processing of massive data and effective extraction of crucial information from an enormous pool of heterogeneous features. In response to these challenges, a promising approach is to exploit online learning methodologies that performs incremental learning over the training samples in a sequential manner. In an online learning algorithm, one sample instance is processed at a time to obtain a simple update, and the process is repeated via multiple passes over the entire training set. In comparison with batch learning algorithms in which all sample points are scrutinized at every single step, online learning algorithms have been shown to be more efficient and scalable for data of large size that cannot fit into the limited memory of a single computer. As a result, online learning algorithms have been widely adopted for solving large-scale machine learning tasks \citep{bottou1998online}.

In this paper, we focus on first-order subgradient-based online learning algorithms, which have been studied extensively in the literature for dense data.\footnote{{\em Dense data} is defined as a dataset in which the number of nonzero entries in all columns of its design matrix are in the order of $O(n)$ while the ones of {\em sparse data} are in the order of $O(\log(n))$ or less. }  Among these algorithms, popular methods include the Stochastic Gradient Descent (SGD) algorithm \citep{zhang2004solving, bottou2010large}, the mirror descent algorithm \citep{beck2003mirror} and the dual averaging algorithm \citep{nesterov2009primal}. Since these methods only require the computation of a (sub)gradient for each incoming sample, they can be scaled efficiently to high-dimensional inputs by taking advantage of the finiteness of the training sample. In particular, the stochastic gradient descent algorithm is the most commonly used algorithm in the literature of subgradient-based online learning. It enjoys an exceptionally low computational complexity while attaining steady convergence under mild conditions \citep{bottou1998online}, even for cases where the loss function is not everywhere differentiable.

 %From the perspective of stochastic optimization where it is assumed that each data point is randomly and independently drawn from an unknown data distribution, the SGD algorithm approximates the expected gradient using the estimation from a subset of data. 

Despite of their computational efficiency, online learning algorithms without further constraint on the parameter space suffers the ``curse of dimensionality'' to the same extent as their non-online counterparts. Embedded in a dense high-dimensional parameter space, not only does the resulted model lack interpretability, its variance is also inflated. As a solution, \textit{sparse online learning} was introduced to induce sparsity in the parameter space under the online learning framework \citep{langford2009sparse}. It aims at learning a linear classifier with a sparse weight vector, which has been an active topic in this area. For most efforts in the literature, sparsity is introduced by applying $L_1$ regularization on a loss function as in the classific LASSO method \citep{tibshirani1996lasso, shalev2011stochastic}. For example, \cite{duchi2009efficient} extend the framework of Forward-Backward splitting \citep{lions1979splitting} by alternating between an unconstrained truncation step on the sample gradient and an optimization step on the loss function with a penalty on the distance from the truncated weight vector.\cite{langford2009sparse} and \cite{carpenter2008lazy} both explore the idea of imposing a {\em soft-threshold} on the weight vector $\mathbf{w}  \in \mathbb{R}^p$ updated by the stochastic gradient descent algorithm:
$$w_j = \mbox{sign}(w_j)\max(|w_j| - \lambda, 0),\quad j=1, \dots, p. $$
This class of methods is known as the \textit{truncated gradient algorithm}. For every $K$ standard SGD updates, the weight vector is shrunk by a fixed amount to induce sparsity. In the work of \cite{duchi2010composite}, the same strategy has also been combined with a variant of the mirror descent algorithm \citep{beck2003mirror}.  \cite{wang2015framework} further extends the truncated gradient framework to adjust for cost-effectiveness. This simple yet efficient method of truncated gradients particularly motivates the algorithm proposed in this paper. Strategies different from the truncation-based algorithm have also been proposed. For example, \cite{xiao2009dual} proposed the Regularized Dual-Averaging (RDA) algorithm which builds upon the primal-dual subgradient method by \cite{nesterov2009primal}. The RDA algorithm learns a sparse weight vector by solving an optimization problem using the running average over all preceding gradients, instead of a single gradient at each iteration.

Closely related to sparse online learning is another area of active research, \textit{online feature selection}. Instead of enforcing just a shrinkage on the weight vectors via $L_1$ regularization, online feature selection algorithms explicitly invoke feature selection by imposing a hard $L_0$ constraint on the weight vector, \citep[e.g.,][]{wang2014online, wu2014large}. In other words, online feature selection algorithms focus on generating a resulted weight vector that has a high sparsity level by directly shrinking a large proportion of the weights directly to zero (also referred to as a {\em hard thresholding}). In practice, $L_0$ regularization is computationally expensive to solve due to its non-differentiability. The set of selected features also suffers from high variability as the decisions of hard-thresholding are based on single random samples in an online learning setting. Therefore, important features can be discarded simply owing to random perturbations. 

Most recent subgradient-based online learning algorithms do not consider potential structures or heterogeneity in the input features. As pointed out by \cite{duchi2011adaptive}, current methods largely follow a predetermined procedural scheme that is oblivious to the characteristics of data being used at each iteration. In large-scale applications, a common and important structure is heterogeneity in sparsity levels of the input features, i.e., the variability in the number of nonzero entries among features. For instance, consider the bag-of-word features in text mining applications.\footnote{Here, by {\em sparse features}, we refer to features for which most samples assume a constant value (e.g., 0) and a few samples take on other values. Without loss of generality, we assume the majority constant is 0 throughout this paper.}  For a learning task, the importance of a feature is not necessarily associated with the frequencies of its values. In genetics, for example, rare variants ($\le 1\%$ in the population) have been found to be associated with disease risks \citep{morris2010evaluation}. Both dense and sparse features may contain important information for the learning task. However, in the presence of heterogeneity in sparsity levels, using a simple $L_1$ regularization in an online setting will predispose rare features to be truncated more than necessary. The resulted sparse weight vectors usually exhibit high variance in terms of both weight values and the membership in the set of features with nonzero weights. As a result, the convergence of the standard truncation-based framework may also be hampered by this high variability. When the amount of information is scarce due to sparsity at each iteration, the convergence of the weight vector would understandably take a large number of iterations to approach the optimum. In two recent papers, \cite{oiwa2011frequency} and \cite{oiwa2012healing} tackle this problem via $L_1$ penalty weighted by the accumulated norm of subgradients for extending several basic frameworks in sparse online learning. Their results suggest that, by acknowledging the sparsity structure in the features, both prediction accuracy and sparsity are improved over the original algorithms while maintaining the same convergence rate. However, their resulted weight vectors are unstable as the imposed subgradient-based regularization are excessively noisy due to the randomness of incoming samples in online learning. The membership in the set of selected features with nonzero weights is also very sensitive to the orderings of the training samples.

%Nonetheless, the majority of existing subgradient-based online learning algorithms do not consider the potential structure or variability in the data. 

%It is often observed in real-world high-dimensional datasets that the number of nonzero entries varies widely among feautres, such as the bag-of-word features in text mining applications.

In this paper, we propose a \textit{stabilized truncated stochastic gradient descent} algorithm for high-dimensional sparse data. The learning framework is motivated by that of the Truncated Gradient algorithm proposed by \cite{langford2009sparse}. To deal with the aforementioned issues with sparse online learning methods applied to high-dimensional sparse data, we introduce three innovative components to reduce variability in the learned weight vector and stabilize the selected features. First, when applying the soft-thresholding, instead of a uniform truncation on all features, we perform only \textit{informative truncations}, based on actual information from individual features during the preceding computation window of $K$ updates. By doing so, we reduce the heterogeneous truncation bias associated with feature sparsity. The key idea here is to ensure that each truncation for each feature is based on sufficient information, and the amount of shrinkage is adjusted for the information available on each feature. Second, beyond the soft-thresholding corresponding to the ordinary $L_1$ regularization, the resulted weight vector is \textit{stabilized} by staged purges of irrelevant features permanently from the active set of features. 
Here, {\em irrelevant features} are defined as features whose weights have been repeatedly truncated. Motivated by \textit{stability selection} introduced in \citet{meinshausen2010stability}, these permanent purges prevent irrelevant features from oscillating between the active and non-active set of features,  The ``\textit{purging}'' process also resembles hard-thresholding in online feature selection and results in a stabler sparse solution than other sparse online learning algorithms. Results on the theoretical regret bound (See Section~\ref{sec: properties}) show that this stabilization step helps improve over the original truncated gradient algorithm, especially when the target weight vector is notably sparse. To attune the proposed learning algorithm to the sparsity of the remaining active features, the third component of our algorithm is adjusting the amount of shrinkage progressively instead of fixing it at a predetermined value across all stages of the learning process. A novel hyperparameter, \textit{rejection rate}, is introduced to balance between {\em exploration} of different sparse combinations of features at the beginning and the {\em exploitation} of the selected features to construct accurate estimate at a later stage. Our method gradually anneal the rejection rate to acquire the necessary amount of shrinkage on the fly for achieving the desired balance.

The rest of paper is organized as follows. Section~\ref{sec_truncated_SGD} reviews the Truncated Gradient algorithm based on Stochastic Gradient Descent (SGD) framework for sparse learning proposed in \citet{langford2009sparse}. In Section~\ref{sec_stabilized_truncated_SGD}, we introduce, in details, the three novel components of our proposed algorithm. Theoretical analysis of the expected online regret bound is given in Section~\ref{sec: properties}, along with the computational complexity. Section~\ref{sec_practical_remarks} gives practical remarks for efficient implementation. In Section~\ref{sec_emp_result}, we evaluate the performance of the proposed algorithm on several real-world high-dimensional datasets with varying sparsity levels. We illustrate that the proposed method leads to improved stability and prediction performance for both sparse and dense data, with the most improvement observed in data with the highest average sparsity level. Section~\ref{sec_conclusion} concludes with further discussion on the proposed algorithm. 

\section{Truncated Stochastic Gradient Descent for Sparse Learning} \label{sec_truncated_SGD}

Assume that we have a set of training data $\mathcal{D} = \{ z_i = (\mathbf{x}_i , y_i), i=1, \dots, n\}$, where the feature vector $\mathbf{x}_i \in \mathbb{R}^p$ and the scalar output $y_i \in \mathbb{R}$. In the following, we use $\mathbf{x}_i$ to represent the vector of the $i^{th}$ sample of length $p$ and $\mathbf{x}_{\cdot, j}$ for the $j^{th}$ feature vector of all samples of length $n$. In this paper, we are interested in the case that both $p$ and $n$ are large and the feature vectors $\mathbf{x}_{\cdot, j}$'s, $j=1, \dots, p$, are sparse. We consider a loss function $l(\hat{y}, y)$ that measures the cost of predicting $\hat{y}$ when the truth is $y$. The prediction $\hat{y}$ is given by function $f_{\mathbf{w}}(\mathbf{x})$ from a family $\mathcal{F}$ parametrized by a weight vector $\mathbf{w}$. Denote $L(\mathbf{w}, \mathbf{z}) \eqdef  l(f_{\mathbf{w}}(\mathbf{x}), y)$. The learning goal is to obtain an optimal weight vector $\hat{\mathbf{w}}$ that minimize the loss function $ \sum\limits_{i=1}^{n} L(\mathbf{w}, z_i)$ over the training data, with sparsity in the weight vector induced by a regularization term $\Psi(\mathbf{w})$. We can then formulate the learning task as a regularized minimization problem: 
\begin{align} \label{eq_obj_1}
\hat{\mathbf{w}} & = \underset{\mathbf{w} \in \mathbb{R}^p}{\operatorname{arg \min}} \ \sum\limits_{i=1}^{n} L(\mathbf{w}, z_i) + \Psi(\mathbf{w}).
\end{align}
The above optimization problem is often solved using some version of {\em gradient descent}. When both $p$ and $n$ are large, the computation becomes very demanding. To address this computational complexity, the Stochastic Gradient Descent (SGD) algorithm was proposed as a stochastic approximation of the full gradient algorithm \cite{bottou1998online}. Instead of computing the gradient over the entire training set as under the batch setting, the stochastic gradient descent algorithm uses approximate gradients based on subsets of the training data. This is particularly attractive to large scale problems as it leads to a substantial reduction in computing complexity and potentially distributed implementation. 

For applications with large data sets or streaming data feeds, SGD has also been used as a subgradient-based {\em online learning} method. Online learning and stochastic optimization are closely related and interchangeable most of the time \citep{cesa2004generalization}. For simplicity,  in the following, we focus our discussion and algorithmic description under the online learning framework with regret bound models. Nonetheless, our results can be readily generalized to stochastic optimization as well.

In online learning, the algorithm receives a training sample $z_t= (\mathbf{x}_t, y_t)$ at a time from a continuous feed. Without sparsity regularization, at time $t$, the weight vector is updated in an online fashion with a single training sample $z_t \in \mathcal{D}$ drawn randomly, 

\begin{align} \label{eq_ord_sgd_update}
\mathbf{w}_{t} &= \mathbf{w}_{t-1} - \eta   L'(\mathbf{w}_{t-1}, z_t)
\end{align}

where $\eta > 0$ is the learning rate and $L'(\mathbf{w}_t, z_t) \in \partial_{\mathbf{w}_t}L(\mathbf{w}_t, z_t)$ is a subgradient of the loss function $L(\mathbf{w}_t, z_t)$ with respect to $\mathbf{w}_t$. The set of subgradients of f at the point $x$ is called the subdifferential of $f$ at $x$, and is denoted $\partial f(x)$. A function $f$ is called subdifferentiable if it is subdifferentiable at all $x \in$ \textbf{dom} $f$. When $L(\mathbf{w}, \cdot)$ is differentiable at $\mathbf{w}$, $\partial_{\mathbf{w}}L(\mathbf{w}, \cdot) = \{\nabla_{\mathbf{w}} L(\mathbf{w}, \cdot) \}$. At the same time, a sequence of decisions $\mathbf{w}_t$ is generated at $t=1, 2, \dots$, that encounters a loss $L(\mathbf{w}_t, z_t)$ respectively. 

The goal of online learning algorithm with sparsity regularization is to achieve \textit{low} regret with respect to a fixed optimal weight vector $\mathbf{w}^* \in \mathcal{W} \subset \mathbb{R}^p$.  Here, $\mathcal{W}$ is the parameter space for sparse weights vectors (see Assumption~\ref{assump_3}  on page 15 for more details.) The regret is defined as:
\begin{equation} \label{eq_regret}
R_T(\mathbf{w}^*)  \triangleq \sum\limits_{t=1}^{T} \left( L(\mathbf{w}_t, z_t) + \Psi(\mathbf{w}_t) \right) - \sum\limits_{t=1}^{T} \left( L(\mathbf{w}^*, z_t) + \Psi(\mathbf{w}^*) \right).
\end{equation}

In this paper, we focus on the $L_1$ regularization where $\Psi(\mathbf{w}) = g||\mathbf{w}||_1$ and $g$ is the regularizing parameter. When adopted in an online learning framework, standard SGD algorithm does not work well in addressing \eqref{eq_obj_1} with $L_1$ penalty. Firstly, a simple online update requires the projection of the weight vector $\mathbf{w}$ onto a $L_1$-ball at each step, which is computationally expensive with a large number of features. Secondly, with noisy approximate subgradient computed using a single sample, the weights can easily deviate from zero due to the random fluctuations in $z_t$'s. Such a scheme is therefore inefficient to maintain a sufficiently sparse weight vector. 

To address this issue, \cite{langford2009sparse} induced sparsity in $\mathbf{w}$ by subjecting the stochastic gradient descent algorithm to soft-thresholding. For every $K \in \mathbb{N}^+$ iterations at step $t$, each of which is as defined in \eqref{eq_ord_sgd_update}, the weight vector is shrunk by a soft-threshold operator $T$ with a \textit{gravity parameter} $\mathbf{g} \in \mathbb{R}^p$ with $g_j \geq 0$ for $j=1, \dots, p$. For a vector $\mathbf{w} = [w_1, \dots, w_p] \in \mathbb{R}^p$,
\begin{equation}\label{eq_weight_trunc}
\hat{\mathbf{w}}_{t} = T(\mathbf{w}_t, \mathbf{g}),
\end{equation}
where $T(\mathbf{w}, \mathbf{g}) = \left[ T(w_1, g_1), \dots, T(w_p, g_p) \right]$ with the operator $T$ defined by 
\begin{align*} 
T(w_j, g_j) &  \triangleq \left\{ \begin{array}{ll} 
\max(w_j - g_j, 0), & \mbox{if } w_j > 0; \\
\min(w_j + g_j, 0), & \mbox{if } w_j \leq 0. \\
\end{array} \right. \numberthis \\  \label{eq_update_trunc}
\end{align*}

As one can see, the sequence of $K$ SGD updates can be treated as a unit computational block, which will be referred to as a \textit{burst} hereafter. Here the word \textit{burst} indicates that it is a sequence of repetitive actions, e.g., the standard SGD updates as defined in \eqref{eq_ord_sgd_update}, without interruption. Each burst is followed by a {\em soft-thresholding truncation} defined in \eqref{eq_weight_trunc}, which puts a shrinkage on the learned weight vector. 

A burst can be viewed as a base feature selection realized on a set of random samples with $L_1$ regularization as in the classical LASSO \citep{tibshirani1996lasso}. Within a burst, let $\mathcal{X}_K$ be the set of $K$ random samples on which the weight vector $\hat{\mathbf{w}}$ is stochastically learned. We define the set of features with nonzero weights in $\hat{\mathbf{w}}$ as its \textit{active (feature) set}:
\begin{equation} \label{eq_selected_set}
\hat{\mathcal{S}}^g(\hat{\mathbf{w}}; \mathcal{X}_K) = \{ j: | \hat{w}_j| > 0 \},
\end{equation}
with a corresponding gravity $\mathbf{g}$. The steps within a truncated burst are summarized in Algorithm~\ref{alg:ssgd_1}.

%basic computation unit that assembles $K$ SGD updates in \eqref{eq: ord_sgd_update} in a sequence. At the end of each burst, it is followed by a truncation in \eqref{eq: weight_trunc}. Within each burst, a repetitive action, the standard SGD update, is taken sequentially without any intervention. At the end of it, a inspection, the soft-threshold truncation, takes a look on the learned weight vector and puts a shrinkage on it. A burst can also be considered as a base feature selection realized by the $L_1$ regularization as the classical LASSO \citep{tibshirani1996lasso}. Given the set of random samples $\mathcal{X}_K$ of size $K$ on which the weight vector is stochastically learned within a burst, we define the set of features with nonzero weights in $\hat{\mathbf{w}}$ as its \textit{active (feature) set}:

In the truncated gradient algorithm of \cite{langford2009sparse}, the gravity parameter is a constant across all dimensions as $\mathbf{g} = g_0 K \mathbf{1}_p$, where $g_0 \in \mathbb{R} \geq 0$ is a \textit{base gravity} for each update in a burst and $\mathbf{1}_p \triangleq (1, \dots, 1) \in \mathbb{R}^p$. In general, with greater parameter $g_0$ and smaller burst size $K$, more sparsity is attained. When $g_0 = 0$, the update in \eqref{eq_weight_trunc} becomes identical to the standard stochastic gradient descent update in \eqref{eq_ord_sgd_update}. \cite{langford2009sparse} showed that this updating process can be regarded as an online counterpart of $L_1$ regularization in the sense that it approximately solves \eqref{eq_obj_1} in the limit  as $K \rightarrow \infty$ and $\eta \rightarrow 0$.

\begin{algorithm}[tb]
   \caption{$B_0(\mathbf{w}_0, g_0)$: A burst of $K$ updates with truncation using a universal gravity as in truncated SGD by \cite{langford2009sparse}.}
   \label{alg:ssgd_1}
\begin{algorithmic}
   \STATE {\bfseries Input:} $\mathbf{w}_0$ at initialization and the base gravity $g_0$.
   \STATE {\bfseries Parameters:} $K$, $\eta$.
   \FOR{$t=1$ {\bfseries to} $K$}
   \STATE Draw $z_t \in \mathcal{D}$ uniformly at random. 
   \STATE  $\mathbf{w}_t = \mathbf{w}_{t-1} - \eta L' \left( \mathbf{w}_{t-1}, z_t \right)$, where $L' \left( \mathbf{w}_{t-1}, z_t \right) \in \partial_{\mathbf{w}_{t-1}} L \left( \mathbf{w}_{t-1}, z_t \right)$.
   % $\mathbf{w}_t = \mathbf{w}_{t-1} - \eta \nabla_{\mathbf{w}_{t-1}} L \left( \mathbf{w}_{t-1}, z_t \right)$.
   \ENDFOR
   \STATE $\hat{\mathbf{w}} = T(\mathbf{w}_K, g_0 K \mathbf{1}_p)$.
   \STATE {\bfseries Return:} $\hat{\mathbf{w}}$.
\end{algorithmic}
\end{algorithm}

\section{Stabilized Truncated SGD for Sparse Learning} \label{sec_stabilized_truncated_SGD}
Truncated SGD \cite{langford2009sparse} works well for dense data. When it comes to high-dimensional {\em sparse} inputs, however, it suffers from a number of issues. \cite{shalev2011stochastic} observe that the truncated gradient algorithm is incapable of maintaining sparsity of the weight vector as it iterates. Recall that, under the online learning setting, the weight vector $\mathbf{w}$ is updated with a noisy approximation of the true expected gradient using one sample at a time, from a random ordering of the data. With sparse inputs, it is highly probable that an important feature does not have a nonzero entry for many consequent samples, and is meaningfully updated for only a few times out of the $K$ updates in a burst.  As a result, it would be truncated after a few iterations and brought back to nonzero after another few updates. At the same time, sparsity in inputs will also give rise to sporadic large nonzero updates for irrelevant features, which cannot be fully resolved by the soft-threshold operator. The derived weight vector $\mathbf{w}$'s are of high variance, inadequate sparsity and poor generalizability.  As an example, the number of nonzero variables in the weight vector during the last 1000 stochastic updates from the truncated gradient algorithm implemented on a high-dimensional sparse dataset (Dexter text mining data set; see Section~\ref{sec_emp_result} for details.) are shown in Figure~\ref{fig_trunc_demo}. It can be seen that the numbers of nonzero features in the weight vectors learned by the truncated SGD algorithm ($K=5$) remain large and highly unstable throughout these 1000 iterations, oscillating within 10\% of the total number of features. As a comparison, also in Figure~\ref{fig_trunc_demo}, we plot the results from our proposed stabilized truncated SGD applied to the same data. During these last 1000 updates, the proposed algorithm is using a less frequent truncation schedule due to our {\em annealed reject rate}. It attains both high sparsity in the weight vector and high stability with high-dimensional sparse data. 

%On the other hand, although the soft-threshold operator may manage to shrink the weights of some irrelevant features to zero at some steps, they might later go up sharply when a considerable number of nonzero entries concentrate in some following steps. 

\begin{figure}[th] \centering
\includegraphics[width=12cm]{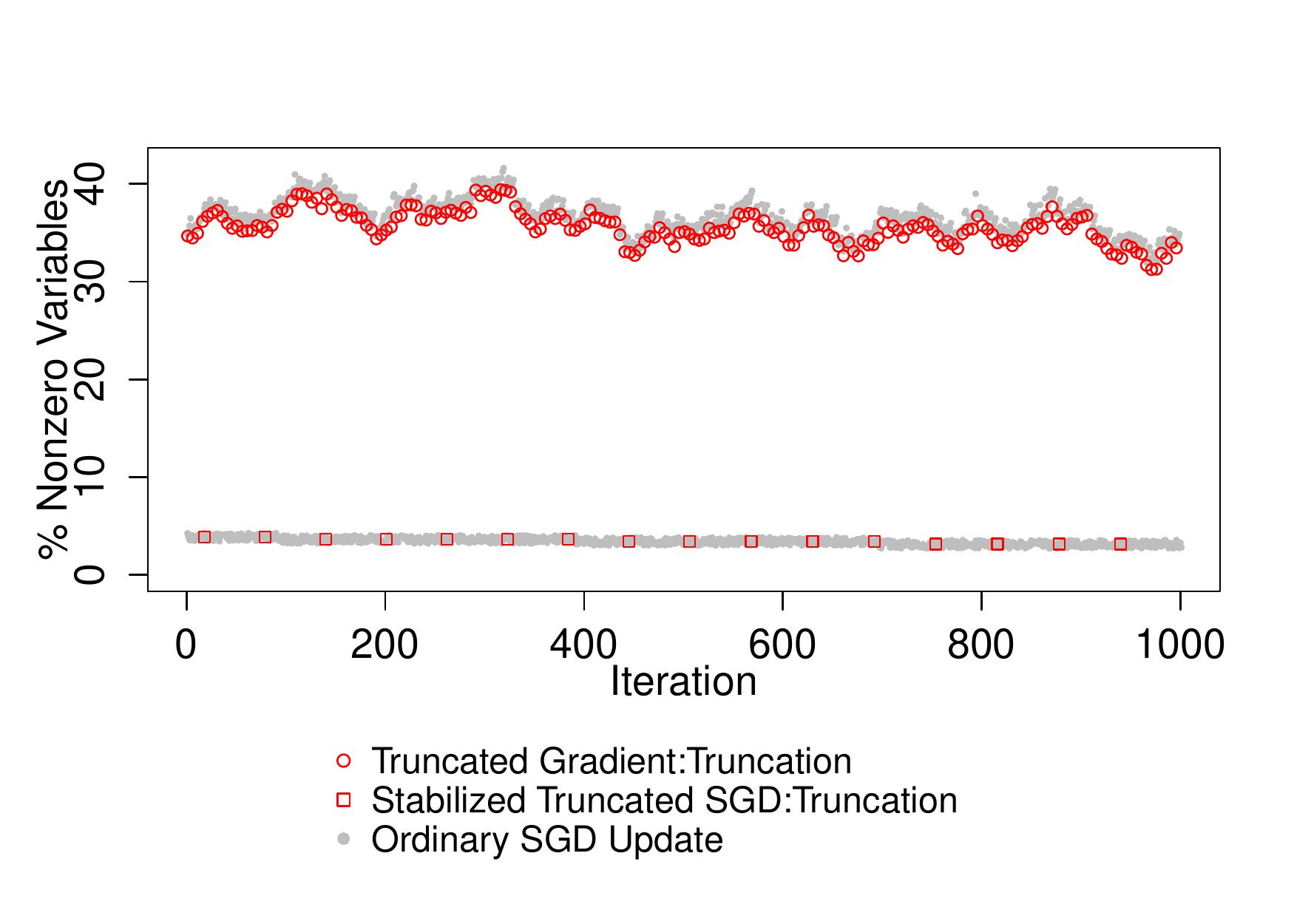}
\caption{An example of the truncated SGD algorithm \citep{langford2009sparse} and the proposed stabilized truncated SGD algorithm applied to a high-dimensional sparse data set. Here we compare the percentage of nonzero variables in the resulted weight vector at each iteration during the last 1000 iterations of both algorithms. The underlying data set is the text mining dataset, {\em Dexter},  with $10,000$ features and 0.48\% of sparsity, which is described in details in Section~\ref{sec_emp_result}. \label{fig_trunc_demo}
} 
\end{figure}

In this section, we introduce the \textit{stabilized truncated Stochastic Gradient Descent (SGD)} algorithm. It attains a truly sparse weight vector that is stable and gives generalizable performance. Our proposed method attunes to the sparsity of each feature and adopts \textit{informative truncation}. The algorithm keeps track of whether individual features have had enough information to be confidently subject to soft-thresholding. Based on the truncation results, we systematically reduce the active feature set by permanently discarding features that are truncated to zero with high probability via \textit{stability selection}. We further improve the efficiency of our algorithm by adapting gravity to the sparsity of the current active feature set as the algorithm proceeds.

\subsection{Informative Truncation} \label{sec_informative_truncation}

For the truncated SGD algorithm, \cite{langford2009sparse} suggest a general guideline for determining gravity in the batch mode by scaling a base gravity $g_0$ by $K$, the number of updates, for a single truncation after a burst. A direct online adaptation of a $L_1$ regularization would shrinks the weight vector at every iteration. The above batch mode operation is to delay the shrinkage for $K$ iterations so that the truncation is executed based on information collected from $K$ random samples instead of from a single instance. This guideline implicitly assumes that the $K$ SGD updates in a burst are equally informative, which is in general true for dense features. For sparse features, however, under the online learning setting, not every update is informative about every feature due to the scarcity of nonzero entries. The original uniform formula, $\mathbf{g}=K g_0 \mathbf{1}_p$,  for gravity would then create an undesirable differential treatment for features with different levels of sparsity. With a relatively small $K$, it is very likely that a substantial proportion of features would have no non-zero values on a size-$K$ subsample used in a particular burst. The weights for these features remain unchanged after $K$ updates. Consequently, the set of sparse features run the risk of being truncated to zero based on very few informative updates. The truncation decision is therefore mostly determined by a feature's sparsity level, rather than its relevance to the class boundary.

%Instead of performing truncation at each iteration, \cite{langford2009sparse} suggest a more aggressive truncation with a gravity parameter proportional to the number of iterations between two truncations $K$ as shown in Algorithm \ref{alg:ssgd_1}, which may lead to better sparsity. Different from its online counterpart of $L_1$ regularization which shrinks the weight vector at every iteration, the mechanism behind this operation is to delay the shrinkage for $K$ iterations so that the truncation is executed based on information collected from $K$ random samples instead of from a single instance. With dense input data, such truncation may works well since, at every iteration, each feature is updated with incremental information about the true gradient when fed with dense vectors. However, the frequencies of feature's nonzero occurrences are usually not uniformly distributed in sparse input data. At each iteration, there would be only a small fraction of features that truly get updated with new information. 

To make the learning be informed of the heterogeneity in sparsity level among features, we introduce the \textit{informative truncation} step, extended from the idea of base gradient used in Algorithm~\ref{alg:ssgd_1}. Instead of applying a universal gravity proportional to $K$ to all features, the amount of shrinkage is set proportional to the number of times that a feature is actually updated with nonzero values in the size-$K$ subsample, i.e., the number of \textit{informative updates}. Specifically, within each burst, the algorithm keeps a vector of \textit{counters}, $\tilde{\mathbf{k}} \in \mathbb{R}^p$, of the numbers of informative updates for the features $\mathbf{x}_{\cdot, j}$, $j=1, \dots, p$. Let $g_0 \in \mathbb{R} \geq 0$ be the base gravity parameter that serves as the unit amount of shrinkage for each informative update on each feature. At the end of each burst, we shrink feature $\mathbf{x}_{\cdot, j}$ by $g_0 \tilde{k}_j$. In other words, here we set $\mathbf{g} = g_0 \tilde{\mathbf{k}}$. The computational steps for a burst with informative truncation in summarized in Algorithm~\ref{alg:ssgd_2}.

\begin{algorithm}[th]
   \caption{$B_1(w_0, g_0)$: A burst of $K$ updates with informative truncation.}
   \label{alg:ssgd_2}
\begin{algorithmic}
   \STATE {\bfseries Input:} $\mathbf{w}_0$ at initialization and the base gravity $g_0$.
   \STATE {\bfseries Initialization:} $\tilde{\mathbf{k}} = \mathbf{0}_{p} \in \mathbb{R}^p$.
   \STATE {\bfseries Parameters:} $K$, $\eta$.
   \FOR{$t=1$ {\bfseries to} $K$}
   \STATE Draw $z_t = (\mathbf{x}_t, y_t) \in \mathcal{D}$ uniformly at random. 
   \STATE $\mathbf{w}_t = \mathbf{w}_{t-1} - \eta  L' \left( \mathbf{w}_{t-1}, z_t \right)$, where $L' \left( \mathbf{w}_{t-1}, z_t \right) \in \partial_{\mathbf{w}_{t-1}} L \left( \mathbf{w}_{t-1}, z_t \right)$.
   %$\mathbf{w}_t = \mathbf{w}_{t-1} - \eta \nabla_{\mathbf{w}_{t-1}} L \left( \mathbf{w}_{t-1}, z_t \right)$.
   \STATE $\tilde{\mathbf{k}} \leftarrow \tilde{\mathbf{k}} + \mathbbm{1}(|\mathbf{x}_t| > 0)$.
   \ENDFOR
   \STATE $\hat{\mathbf{w}} = T(\mathbf{w}_K, g_0 \tilde{\mathbf{k}})$.
   \STATE {\bfseries Return:} $\hat{\mathbf{w}}$, $\tilde{\mathbf{k}}$.
\end{algorithmic}
\end{algorithm}

The proposed informative truncation scheme ensures that the decision of truncation is made based on sufficient and equivalent amount of evidence for evaluating each feature. A theoretical justification of how informative truncation helps improving truncation bias can be found in Lemma~\ref{lemma2} (Section~\ref{sec: properties}). This feature-specific gravity attunes to the sparsity structure incurred at each burst without \textit{ad-hoc} adjustment. It also avoids data pre-processing for locating sparse entries, which can be computationally expensive and compromises the advantage of online computation. In comparison to the truncated gradient algorithm in \citet{langford2009sparse} that quickly shrinks many features to zero indiscriminately, informative truncation keeps sparse features until enough evaluation is conducted. In doing so, sparse yet important features will be retained. The proposed approach also reduce the variability in the resulted sparse weight vector during the training process. \cite{duchi2011adaptive} uses a similar strategy that allows the learning algorithm to adaptively adjust its learning rates for different features based on cumulative update history. They use the $L_2$ norm of accumulated gradients to regulate the learning rate. By adapting the gravity with the counter $\tilde{\mathbf{k}}$ within each burst, our proposed strategy here can be viewed as applying the $L_0$ norm to the accumulated gradients that is refreshed every $K$ steps.

\subsection{Stability Selection}

Despite of its scalability, subgradient-based online learning algorithms commonly suffer from instability. It has been shown both theoretically and empirically that stochastic gradient descent algorithms are sensitive to random perturbations in training data as well as specifications of learning rate \citep{toulis2015stability, hardt2015train}. This instability is particularly pronounced in sparse online learning with sparse data, as discussed in Section~\ref{sec_intro}. Under an online learning setting, using random ordering of the training sample as inputs, the algorithm would produce distinct weight vectors and unstable memberships of the final active feature set. Moreover, there has been a lot of discussion, in the literature,  on the link between the instability of an learning algorithm and  its deteriorated generalizability \citep{bousquet2002stability, kutin2002almost, rakhlin2005stability, shalev2010learnability}. 

To tackle this instability issue, in the proposed algorithm, we exploit the method of {\em stability selection} to improve its robustness to random perturbation in the training data. Stability selection \citep{meinshausen2010stability} does not launch a new feature selection method. Rather, its aim is to enhance and improve a sparse learning method via subsampling. The key idea of stability selection is similar to the generic bootstrap \citep{meinshausen2010stability}. It feeds the base feature selection procedure with multiple random subsamples to derive an empirical selection probability. Based on aggregated results from subsamples, a subset of features is selected with low variability across different subsamples. With proven consistency in variable selection, stability selection helps remove noisy irrelevant features and thus reduce the variability in learning a sparse weight vector.

Incorporating stability selection into our proposed framework, each truncated burst with gravity parameter $\mathbf{g}$ is treated as an individual sparse learning engine. It takes $K$ random samples and carries out a feature selection to obtain a sparse weight vector.  In the following, we define first the notion of {\em selection probability} for the stability selection step in our proposed algorithm. 

\begin{defn}[selection probability]
Let $\mathcal{X}_K$ be a random subsample of $\{1, \dots, n\}$ of size $K$, drawn without placement. Parametrized by the gravity parameter $\mathbf{g}$, the probability of the feature $\mathbf{x}_{\cdot, j}$ being in the active set of a truncated burst that returns $\hat{\mathbf{w}}$ is
$$
\Pi^{\mathbf{g}}_j = P^* \left( j \in \hat{S}^{\mathbf{g}} (\hat{\mathbf{w}}; \mathcal{X}_K) \right)= \mathbb{E}_{\mathcal{D}} \left[ \mathbbm{1} (|\hat{w}_j| > 0) \right],
$$
where the probability $P^*$ is with respect to the random subsampling of $\mathcal{X}_K$.  Let $\Pi^{\mathbf{g}} = [\Pi^{\mathbf{g}}_1, \dots, \Pi^{\mathbf{g}}_p]$.
\end{defn}

For simplicity, we drop the superscript $\mathbf{g}$ of $\Pi^{\mathbf{g}}$ in later discussions. For the rest of the paper, the selection probability $\Pi$ always refers to $\Pi^{\mathbf{g}}$ that corresponds to weight vector $\hat{\mathbf{w}}$ with gravity parameter $\mathbf{g}$.

Under unknown data distribution, the selection probabilities cannot be computed explicitly. Instead, they are estimated empirically. Since each truncation burst performs a screening on all features, the frequency of each feature being selected by a sequence of bursts can be used to derive an estimator of the selection probability. We denote a sequence of $n_K > 0$ truncated bursts as a \textit{stage}. A preliminary empirical estimate of the selection probability is given by
\begin{equation} \label{eq_sel_prob_1}
\hat{\Pi}_j = \left\{ \begin{array}{ll}
\frac{\sum_{\tau: \tilde{k}_{j,\tau} > 0}  \mathbbm{1}(|\hat{\mathbf{w}}_{j,\tau} | > 0)}{\sum\limits_{\tau=1}^{n_k} \mathbbm{1}( \tilde{k}_{j, \tau} > 0) }, & \mbox{for $j$ s.t.}  \sum\limits_{\tau=1}^{n_K} \tilde{k}_{j,\tau} > 0 \\
1, &  \mbox{otherwise}
\end{array}
\right.,
\end{equation}
where $\tilde{\mathbf{k}}_{\tau}$ are the counters of informative updates for burst $\tau$, $\tau = 1, \dots, n_K$.

Different from the conventional stability selection setting, $\hat{\mathbf{w}}_{\tau}$'s are obtained sequentially and thus are dependent with each other. When $n_K$ is small, different subsamples produce selection probability estimates using \eqref{eq_sel_prob_1} exhibit high variability, even when initialized with the same weight vector at $\tau=1$. On the other hand, a large value of $n_K$ requires a prohibitively large number of iterations for convergence. To resolve the issues of estimating selection probability using a single sequence of SGD updates, we introduce a multi-thread framework of updating paths. Multiple threads of sequential SGD updates are executed in a distributed fashion, which readily utilizes modern multi-core computer architecture. With $M$ processors, we initialize the algorithm on each path of SGD updates with a random permutation of the training data, $\mathcal{D}$, denoted as $\mathcal{D}^{(1)}, \dots, \mathcal{D}^{(m)}$. Then independently, $M$ stages of bursts run in parallel along $M$ paths, which return with $\hat{\mathbf{w}}_{\tau}^{(m)}$, $\tau = 1, \dots, n_K$, $m=1, \dots, M$. The joint estimate of selection probability with gravity $\mathbf{g}$ is obtained as
\begin{equation} \label{eq_sel_prob_2}
\hat{\Pi}_j = \left\{ \begin{array}{ll}
\frac{ \sum\limits_{m=1}^{M}  \sum_{\tau: \tilde{k}^{(m)}_{j,\tau} > 0}  \mathbbm{1}(|\hat{\mathbf{w}}^{(m)}_{j,\tau} | > 0)}{ \sum\limits_{m=1}^{M}  \sum\limits_{\tau=1}^{n_k} \mathbbm{1}( \tilde{k}^{(m)}_{j, \tau} > 0) }, & \mbox{for $j$ s.t.}  \sum\limits_{m=1}^{M} \sum\limits_{\tau=1}^{n_K} \tilde{k}_{j,\tau}^{(m)} > 0 \\
1, & \mbox{otherwise}
\end{array}
\right..
\end{equation}

When more processors are available, a smaller $n_K$ is required for the algorithm to obtain a stable estimate of selection probability. The dependence among $\hat{\mathbf{w}}_{\tau}$'s is also attenuated when $M$ random subsets of samples are used for the estimation. This strategy falls under parallelized stochastic gradient descent methods, which is discussed in detail by \cite{zinkevich2010parallelized}.

Under the framework of stability selection, each stage on every path uses a random subsample. The estimated selection probability quantifies the chance that a feature is found to have high relevance to class differences given a random subsample. At the end of each stage, \textit{stable features} are identified as those that belong to a large fraction of active sets incurred during this stage of bursts.
\begin{defn}[Stable Features]
For a purging threshold $\pi_0 \in [0,1]$, the set of stable features with gravity parameter $\mathbf{g}$ is defined as 
\begin{align}\label{eq_stable_set}
\hat{\Omega}^{\mathbf{g}} = \{j: \Pi_j^{\mathbf{g}} \geq \pi_0 \}.
\end{align}
\end{defn}
For simplicity, we write the stable set  $\hat{\Omega}^{\mathbf{g}}$ as $\hat{\Omega}$ when there is no ambiguity.

Stability selection retains features that have high selection probabilities and discard those with low selection probabilities. At the end of a stage of $M$ paths, we \textit{purge} the features that are not in the set of stable features by permanently setting their corresponding weights to zero, and remove them from subsequent updates. We define the \textit{stabilized} weight vector as 
\begin{equation} \label{eq_stability_selection}
\tilde{\mathbf{w}} = \hat{\mathbf{w}} \cdot \mathbbm{1}_{\hat{\Omega}}.
\end{equation}

As discussed above, due to the nature of online learning with sparse data, there are two undesirable learning setbacks in a single truncated burst. The first occurs when an important feature has its weight stuck at zero due to inadequate information in the subsample used, while the second case is when a noise feature's weight gets sporadic large updates by chance. Using informative bursts, we can avert the first type of setbacks and using selection probability based on multiple bursts, we can spot noisy features more easily. In the presence of a large number of noisy features, the learned weights for important features suffer from high variance. Via stability selection, we systematically remove noisy features permanently from the feature pool. Furthermore, the choice of a proper regularization parameter is crucial yet known to be difficult for sparse learning, especially due to the unknown noise level. Applying stability selection renders the algorithm less sensitive to choice of the base gravity parameter $g_0$ in learning a sparse weight vector via truncated gradient.  As we will show using results from our numerical experiments, this purging by stability selection leads to a notable reduction in the estimation variance of the weight vector. Here, $\pi_0$ is a tuning parameter in practice. We have found that the learning results in the numerical experiments are not sensitive to different values of $\pi_0$ within a reasonable range. Under mild assumptions discussed in Section~\ref{sec: properties}, we derive a lower bound of the expected improvement in convergence by employing stability selection in the learning process in Lemma~\ref{lemma1}.

%With a single truncated burst, it is very likely that the weight of a truly significant feature is shrunk to zero due to its absence in a randomly chosen $K$ samples and that a noisy irrelevant feature remains nonzero due to its non-negligible nonzero appearance among the random subset when compared to other features. With the joint estimation of selection probability, the stability selection removes noisy features accordingly with greater accuracy. The permanent purging effectively reduces the size of feature subsets to be explored in subsequent updates. The variance of the resulted sparse weight vector is also markedly diminished, which will be shown in our numerical experiments in Section \ref{sec_emp_result}. The exact threshold $\pi_0$ is a tuning parameter but the results vary surprisingly little for sensible choices in a range of the threshold in numerical experiments. Neither do the results depend strongly on the choice of the base gravity parameter $g_0$. 

\begin{algorithm}[tb]
   \caption{$B_2(\mathbf{w}_0, g_0, \hat{\Omega}, \mathcal{D}^{(m)})$: Informative truncated burst with stability selection in thread $m$.}
   \label{alg:ssgd_3}
\begin{algorithmic}
   \STATE {\bfseries Input:} $\mathbf{w}_0$, $g_0$, the input data $\mathcal{D}^{(m)}$ and the current set of stable features $\hat{\Omega}$, which is the output of equation \eqref{eq_stable_set} using \eqref{eq_sel_prob_2} with predetermined threshold $\pi_0$.
   \STATE {\bfseries Parameters:} $K$, $\eta$.
	\STATE Initialize $\mathbf{v}_0 = (w_{0, j})_{j \in  \hat{\Omega}}$ , $\tilde{\mathbf{k}} = \mathbf{0}_{|\hat{\Omega}|}$.
   \FOR{$t=1$ {\bfseries to} $K$}
   \STATE Draw $z_t=(\mathbf{x}_t, y_t)$ sequentially from $\mathcal{D}^{(m)}$. 
   \STATE $\tilde{z}_t = (z_{t,j})_{j \in  \hat{\Omega} }$.
   \STATE $\mathbf{v}_t = \mathbf{v}_{t-1} - \eta L' \left( \mathbf{v}_{t-1}, \tilde{z}_t \right)$, , where $L' \left( \mathbf{v}_{t-1}, z_t \right) \in \partial_{\mathbf{v}_{t-1}} L \left( \mathbf{v}_{t-1}, z_t \right)$.
   \STATE $\tilde{\mathbf{k}} \leftarrow \tilde{\mathbf{k}} + \mathbbm{1}(|\mathbf{x}_t| > 0)$.
   \ENDFOR
   \STATE $\hat{\mathbf{u}} = T(\mathbf{v}_K, g_0 \tilde{\mathbf{k}})$.
   \STATE $\hat{w}_j = \left\{ \begin{array}{ll} 
   \hat{u}_{j'}, & \mbox{if } j \in  \hat{\Omega} \mbox{ and } \hat{\Omega}_{j'} = j \\
   0, &  \mbox{if } j \notin \hat{\Omega} 
   \end{array}    \right.$.
   \STATE {\bfseries Return:} $\hat{\mathbf{w}}$, $\tilde{\mathbf{k}} $.
\end{algorithmic}
\end{algorithm}

\subsection{Adaptive Gravity with Annealed Rejection Rate}

The truncated SGD \citep{langford2009sparse} adopts a universal and fixed base gravity parameter at all truncations. As pointed out in  \cite{langford2009sparse} , a large value of the base gravity $g_0$ achieves more sparsity but the accuracy is compromised, while a small value of $g_0$ leads to less sparse weight vector yet attaining better performance. In other words, different extents of shrinkage serve different purposes of a learning algorithm. The needs for shrinkage also changes as the weight vector and the stable set evolves. Intuitively, the truncation is expected to be greedy at the beginning so that the number of nonzero feature can be quickly reduced for better computational efficiency and learning performance. As the algorithm proceeds, fewer features remain in the stable set. We should then be careful not to shrink important features with a truncation that is too harsh. 

A large base gravity $g_0$ is effective in inducing sparsity at the beginning of the algorithm when the weight vector $\hat{\mathbf{w}}$ is dense. As the algorithm proceeds, the same value of gravity is likely to impose too much shrinkage when the learned weight vector $\hat{\mathbf{w}}$ becomes very sparse, exposing some truly important features at the risk of being purged. On the other hand, a small fixed gravity is over-conservative so that the algorithm will not shrink irrelevant features effectively, leading to slow convergence and a dense weight vector overridden by noise. Tuning a reasonable fixed base gravity parameter for a particular data set does not only creates additional computational burden, but also inadequate in addressing different learning needs during different stages of the algorithm. 

As the role of {\em gravity} in a learning algorithm is to induce sparse estimates, in this paper, we propose an \textit{adaptive gravity} scheme that delivers the right amount of shrinkage at each stage of the algorithm towards a desirable level of sparsity for the learned weight vector. We propose to control sparsity by a target {\em rejection rate} $\beta$, that is, the proportion of updates that are expected to be truncated. Guided by this target rejection rate, we derive the necessary shrinkage amount and the corresponding gravity. As we discussed in Section~\ref{sec_informative_truncation}, a base gravity $g_0$ is used in our learning algorithms to create gravity values for individual features that are attuned to their data sparsity levels. Therefore our adaptive gravity scheme is carried out by adjusting $g_0$. At the beginning of a particular stage, we examine the truncations carried out during the previous stage. The base gravity $g_0$ is then adjusted to project the target rejection rate during the current stage. Specifically, at stage $s$, we look at the pooled set of non-truncated weight vectors and informative truncation counters $\left\{ \mathbf{w}_{\tau}^{(m)},  \tilde{\mathbf{k}}_{\tau}^{(m)}, \tau=1, \dots, n_K, m=1, \dots, M \right\}$ from all the bursts conducted in the previous stages on multiple threads. The adaptive  base gravity $g_0$ for a target rejection rate $\beta_s \in [0,1]$ is then obtained as 
\begin{equation} \label{eq_adaptive_g}
g_{0, s}(\beta_s) \triangleq \sup \{ g_0 \geq 0: \hat{p}_s(g_0) \leq \beta_s \}.
\end{equation}
Here $\hat{p}_s(g_0)$ is the empirical probability, i.e.,
\begin{align*}
\hat{p}_s(g_0) \triangleq \frac{\sum\limits_{m=1}^{M} \sum\limits_{\tau=1}^{n_k} \sum_{ \{j: j \in \hat{\Omega}_s, \tilde{k}_{j,\tau}^{(m)} > 0 \} } \mathbbm{1} \left( \left\vert \frac{ \Delta w_{j,\tau}^{(m)}}{\tilde{k}_{j,\tau}^{(m)}} \right\vert > g_0 \right) }{\sum\limits_{m=1}^{M} \sum\limits_{\tau=1}^{n_k} \sum_{j \in \hat{\Omega}_s} \mathbbm{1} \left( \tilde{k}_{j,\tau}^{(m)} > 0 \right)},
\end{align*}
where $\Delta w_{j,\tau}^{(m)} \triangleq w_{j,\tau}^{(m)}  - w_{j,\tau-1}^{(m)} $ is the amount of updates on feature $\mathbf{x}_{\cdot, j}$ during the $\tau^{th}$ burst. 

%[thesis version]
%In other words, we pool the updates in weight vectors learned from previous iterations within one stage and concatenate the scaled weights in the stable set with nonzero counter values as a single vector $\tilde{\mathbf{v}}_s $ of length $d_{\tilde{\mathbf{v}}_s}$, where 
%
%\begin{equation} \label{eq_defn_v}
%\tilde{\mathbf{v}}_s = \left( \frac{\Delta w_{j, \tau}^{(m)}}{\tilde{k}_{j,\tau}^{(m)}} \right)_{ \left\{ j \in \{j': j' \in  \hat{\Omega}_s \mbox{ and } \tilde{k}_{j,\tau}^{(m)} > 0\}, \tau=1, \dots, n_K, m=1, \dots, M \right\} }.
%\end{equation}
%The base gravity is set to be the $\beta_s$ percentile of $\tilde{\mathbf{v}}_s$. Since the vector $\tilde{\mathbf{v}}_s$ is composed of discrete values, the adaptive base gravity $g_{0, s}$ can also be written as the $l^{th}$ order statistics of $\tilde{\mathbf{v}}_s$ where $l = \beta_s d_{\tilde{\mathbf{v}}_s}$.

%With adaptive gravity parameters, we drop the subscript $g$ in notions of the stability probability $\hat{\Pi}$ and of the set of stable features $\hat{\Omega}$. The steps within one truncated burst with adaptive gravity are summarized in Algorithm \ref{alg:ssgd_2}.

We initialize the algorithm with a high rejection rate so that a large proportion of the weight vector can be reduced to zero at the end of each burst during the early stage of the algorithm. It allows the algorithm to explore as many sparse combination of features as possible at the early stage of the learning process. Along with the stability selection, the set of stable features can be quickly reduced to a manageable size by removing the majority of noises. When the weight vector becomes sparse, we decrease the rejection rate proportionally. With a lower rejection rate, and consequently a lower gravity, the algorithm can better exploit the subsequent standard SGD updates for a more accurate estimate of the true weight vector. As the rejection rate decreases to 0, the algorithm converges to the standard stochastic gradient descent algorithm on a small subset of stable features.

To achieve the balance between exploration and exploitation, we construct an annealing function for the rejection rate that decreases monotonically as the level of sparsity decreases. Let $\beta_0 \in [0,1]$ be the \textit{maximum rejection rate} at initialization and let $\gamma$ be the \textit{annealing rate}. The annealing function $\phi$ for the rejection rate at stage $s+1$ is given by
\begin{align*} 
\beta_{s+1} & = \phi(d_s; \beta_0, \gamma)  \\
& = \left\{ \begin{array}{ll}
\beta_0 \left[ \exp \left( - \gamma d_s \right) - d_s e^{-\gamma} \right]   & \gamma \geq 0 \\ \numberthis \label{eq_rej_rate}
\beta_0 \frac{\log(1 -\gamma (1- d_s))}{\log(1-\gamma)}, & \gamma < 0 
\end{array}
\right.,
\end{align*}
where $d_s = \frac{| \hat{\Omega}_{s} |}{p}$ is the level of weight vector sparsity at the end of stage $s$. The greater the value $\gamma$ is, the faster the rejection rate is annealed to zero as the number of stable features decreases. 

A positive, zero and negative value of $\gamma$ corresponds to exponential decay, linear decay and logarithmic decay of the rejection rate, respectively. Figure~\ref{fig_demo_anneal_beta} presents examples of the rejection rate anneal function with $\gamma = -5, 0$, and $5$ respectively.

\begin{figure}[h] 
\centering
\includegraphics[width=8cm]{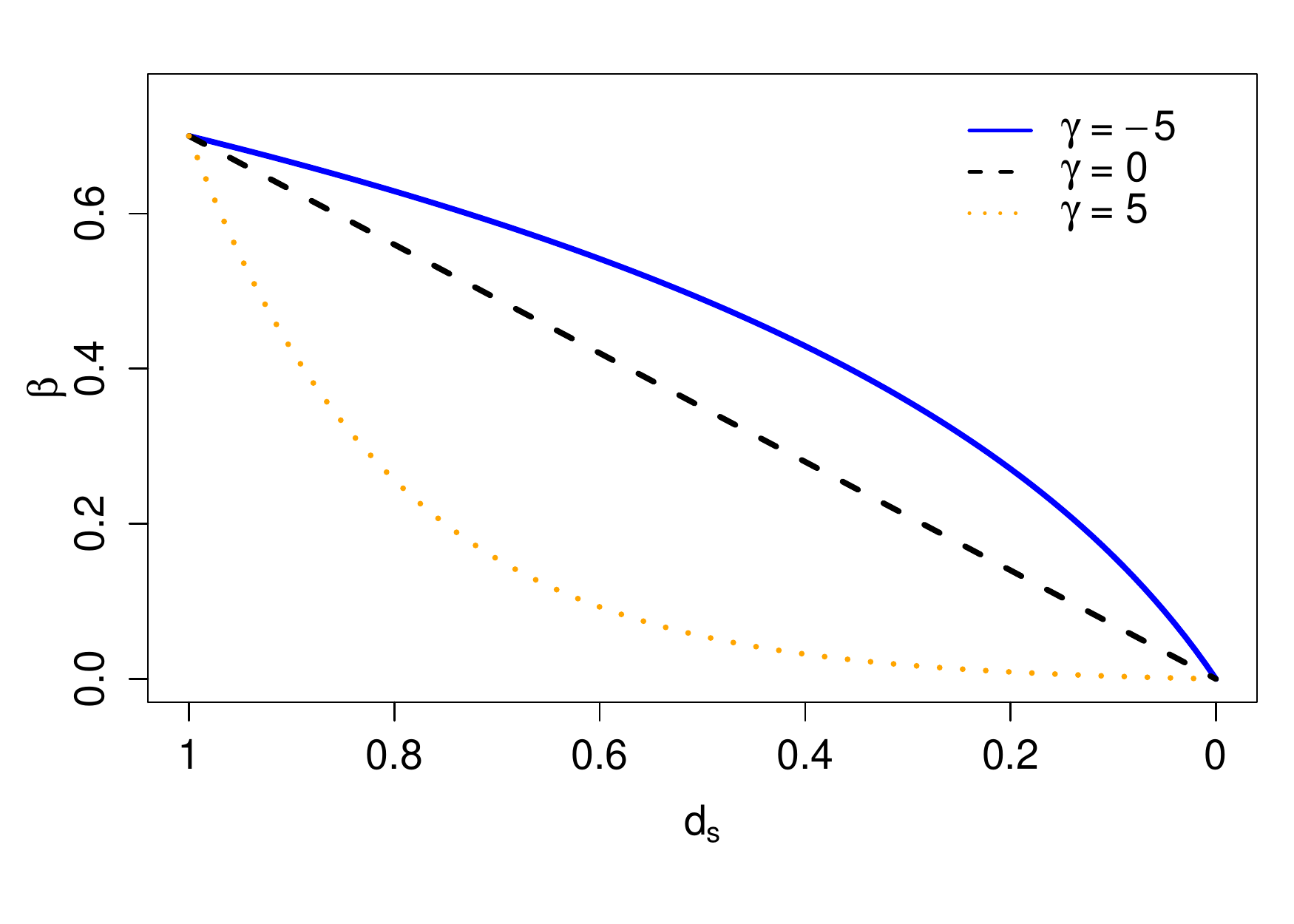}
\caption{Examples of the rejection rate annealing function with different values of $\gamma$ as defined in \eqref{eq_rej_rate}. Here $\gamma = -5, 0$, and $5$ respectively. A positive annealing rate would reduce the rejection rate quickly as the proportion of non-zeros weight values, $d_s$, decreases,  whereas a negative annealing rate would maintain it at a relatively high level.}\label{fig_demo_anneal_beta}
\end{figure}

\begin{algorithm}[tb]
   \caption{$\Psi(\mathcal{D})$: Stabilized truncated stochastic gradient descent for sparse learning.}
   \label{alg:stable_ssgd}

\begin{algorithmic}
	\STATE{\bfseries Input:} the training data $\mathcal{D}$.
   \STATE{\bfseries Parameters:} $\beta_0$, $\gamma$, $\pi_0$, $K$, $n_K$, $M$, $\eta$.
   \STATE{\bfseries Initialization:} For each $m \in \{1, \dots, M\}$, initialize $\tilde{\mathbf{w}}_0^{(m)} = \mathbf{0}_p$, $\tilde{\mathbf{k}} =  \mathbf{0}_p$ and $\mathcal{D}^{(m)}$ is a random permutation of $\mathcal{D}$; $\hat{\Omega}_0 = \{1, \dots, p\}$.
   \STATE {\bfseries For} $s = 1, 2, \dots$,
   \STATE Set $\hat{\mathbf{w}}_{0}^{(m)} =  \tilde{\mathbf{w}}_{s-1}^{(m)}$
   \REPEAT %
   		\STATE Obtain $g_{0,s}(\beta_s)$ as in \eqref{eq_adaptive_g}. 
   		\FORALL {$ m \in \{1, \dots, M\} $ {\bfseries parallel}} 
   			 \FOR{$\tau=1$ {\bfseries to} $n_K$}
   				 \STATE $\hat{\mathbf{w}}_{\tau}^{(m)} = B_2 \left(\hat{\mathbf{w}}_{\tau-1}^{(m)}, g_{0,s}(\beta_s), \hat{\Omega}_{s-1}, \mathcal{D}^{(m)} \right)$.
  			 \ENDFOR
   	    \ENDFOR
   		\STATE Compute $\hat{\Pi}_s$ as in \eqref{eq_sel_prob_2} and update the set of stable features $\hat{\Omega}_s = \left\{ j: \hat{\Pi}_{s,j} \geq \pi_0 \right\}$.
  		%\STATE $\hat{\mathbf{w}}_s^{(m)} =  \hat{\mathbf{w}}_{n_K}^{(m)}$, $m = 1, \dots, M$.
   		\STATE $\tilde{\mathbf{w}}_s^{(m)} = \hat{\mathbf{w}}_{n_K}^{(m)} \mathbbm{1}_{\hat{\Omega}_s}$, $m = 1, \dots, M$.
   		\STATE $\beta_{s+1} = \phi(d_{s}; \beta_0, \gamma)$ where $d_{s} = \frac{|\hat{\Omega}_{s}|}{p}$.
   		\STATE{\bfseries Aggregation:} $\underline{\mathbf{w}} = \frac{1}{M} \sum\limits_{m=1}^{M} \tilde{\mathbf{w}}_s^{(m)}$. 
   \UNTIL{$\underline{\mathbf{w}}$ converges}.
   \STATE {\bfseries Return:} $\underline{\mathbf{w}} $.
\end{algorithmic}
\end{algorithm}

By using adaptive gravity  \eqref{eq_adaptive_g} with annealed rejection rate \eqref{eq_rej_rate}, the amount of shrinkage is adjusted to the current level of sparsity of the weight vector quantified by the size of the stable set $|\hat{\Omega}_s|$ or the $L_0$ norm of the purged $\tilde{\mathbf{w}}$. Instead of tuning a fixed gravity parameter as in \cite{langford2009sparse}, for our proposed algorithm, we tune the annealing rate $\gamma$ and the maximum rejection rate $\beta_0$. Here $\gamma$ balances the trade-off between exploration and exploitation and $\beta_0$ determines the initial intensity of truncation. It enables the tuning process to be tailored to the data at hand as well as being comparable across different datasets. In Section~\ref{sec_emp_result}, the tuning results instantiate that a negative annealing rate is preferred for highly sparse data, such as the RCV1 dataset, since the a high rejection rate needs to be maintained longer allowing sufficient information of sparse features to be evaluated by the learning process. On the other hand, a positive annealing rate is chosen for relatively dense data, such as the Arcene dataset, where the high frequency of nonzero values permit fast reduction of the active set. The complete algorithm of the stabilized truncated stochastic gradient descent algorithm is summarized in Algorithm~\ref{alg:stable_ssgd}.

%[after \tilde{w}, before In
%When the learned weight vector is dense, $|\hat{\Omega}_s|$ is large and thus the base gravity $g_0$ is large in comparison with amount of updates in $\mathbf{w}$, and thus more features are shrunk to quickly attain a sparse weight vector. When $\mathbf{w}$ already achieves high level of sparsity, the base gravity $g_0$ approaches to the minimum of $w_j$'s and thus fewer features are shrunk to zero.

\section{Properties of the Stabilized Truncated Stochastic Gradient Descent Algorithm} \label{sec: properties}

The learning goal of sparse online learning is to achieve a low regret as defined in \eqref{eq_regret}. In this section, we analyze the online regret bound of the proposed stabilized truncated SGD algorithm in Algorithm~\ref{alg:stable_ssgd} with convex loss. For simplicity, the effect of adaptive gravity with annealed rejection rate is not considered here. To achieve viable result, we make the following assumptions.

\begin{assumption} \label{assump_1}
The absolute values of the weight vector $\mathbf{w}$ are bounded above, that is, $|w_j| \leq C$ for some $C \in [0, \infty)$, $j=1, \dots, p$. 
\end{assumption}

\begin{assumption} \label{assump_2}
The loss function $L(\mathbf{w}, z)$ is convex in $\mathbf{w}$, and there exist non-negative constants $A$ and $B$ such that, for all $\mathbf{w} \in \mathbb{R}^p$ and $z \in \mathbb{R}^{p+1}$, $||\nabla_{\mathbf{w}} L(\mathbf{w}, z)||^2 \leq AL(\mathbf{w}, z) + B$.  
\end{assumption}

For linear prediction problems, the class of loss function that satisfies Assumption 2 includes common loss functions used in machine learning problems, such as the $L_2$ loss, the hinge loss and the logistic loss, with the condition that $\sup_{\mathbf{x}} ||x|| \leq C_x$ for some constant  $C_x > 0$.

\begin{assumption} \label{assump_3}
Assume that the set $\mathcal{W}$ for identifying $\mathbf{w}^*$ has the following properties:
\begin{enumerate}
\item $\mathcal{W}$ is the parameter space for weight vectors $\mathbf{w}$  that is subject to a sparsity constraint $$ ||\mathbf{w}||_0 = d^*. $$
For the optimal weight vector $\mathbf{w}^*$, we denote $\Omega^* = \{j: |w^*_j| > 0\}$ and $d^* = |\Omega^*|$.
We further assume that $d^*$ is sufficiently small and the gravity parameter $g$ associated with $\mathbf{w}^*$ is reasonably large so that, for $\tau = 1, 2, \dots$, the average number of active selected features from each truncated bursts, $q^{\mathbf{g}}$, given a gravity $\mathbf{g}$, is greater than or equal to $d^*$.

\item Assume features $\mathbf{x}_1, \dots, \mathbf{x}_p$ has various sparsity distribution that $\Pr(|x_{i,j}| > 0) = \lambda_j$, where $\lambda_j \in [0,1]$, for $i = 1, \dots, n$, and $\mathbf{w}^*_j = 0$ if $\lambda_j = 0$, for $j=1, \dots, p$. 
\end{enumerate}
\end{assumption}

Assumption 3 posits that $\mathbf{w}$'s parameter space of interest is substantially sparse, which is the main focus of sparse learning and of this paper. The theoretical analysis in the following concerns for an fixed optimal weight vector $\mathbf{w}^*$, if exists, under such a constraint. Nevertheless, this condition does not confine the applicable scenarios of the proposed method to a fixed subclass of problems. It suggests a balance between the model sparsity and the value of the gravity parameter that is implicitly embedded within the parameter tuning process.

\begin{lemma}  \label{lemma1} 
Let $\hat{\mathbf{w}}$ be a non-stabilized dense weight vector, i.e., an output weight vector from Algorithm~\ref{alg:ssgd_3}. Let $\tilde{\mathbf{w}}$ be the stabilized weight vector derived from $\hat{\mathbf{w}}$, which is purged by the stability selection \eqref{eq_stability_selection} with a set of stable features $\hat{\Omega}$. Let $q^{\mathbf{g}}$ be the average number of nonzero entries in $\hat{\mathbf{w}}$'s of the previous truncated bursts, i.e.,the number of selected features from $n_k$ truncated bursts, with gravity $\mathbf{g}$. Then, if Assumption 1 holds, there exists an $\varepsilon \in \left( 0, \frac{\pi_0 C(p - |\hat{\Omega}|)}{|\hat{\Omega}|} \right)$ with $\hat{S}_{\varepsilon} = \{j: \mathbb{E} (|\hat{w}_j |) > \varepsilon \}$ such that the bound on the expected difference between the distance from the non-stabilized weight vector to $ \mathbf{w}^*$ under Assumption~\ref{assump_3} and the distance from the stabilized weight vector to $\mathbf{w}^*$ is given by

\begin{align} \label{eq: lemma1} 
\mathbb{E} \left( ||\hat{\mathbf{w}} - \mathbf{w}^*||^2 - ||\tilde{\mathbf{w}} - \mathbf{w}^* ||^2  \right)  & \geq \  \varepsilon^2 ( |\hat{S}_{\varepsilon}| - |\hat{\Omega} | ) + 2 \pi_0 C^2 \left[  \left( 1 - \frac{q^{\mathbf{g}}}{2\pi_0 p - p} \right) q^{\mathbf{g}} - d^* \right] 
\end{align}
When the purging threshold is sufficiently high such that $\pi_0 \in \left( \frac{1}{2} + \frac{(q^{\mathbf{g}})^2}{2p(q^{\mathbf{g}} - d^*)} , 1\right)$, 
$$\mathbb{E} \left( ||\hat{\mathbf{w}} - \mathbf{w}^*||^2 - ||\tilde{\mathbf{w}} - \mathbf{w}^* ||^2  \right)  \geq 0. $$
\end{lemma}

\begin{proof}:
See Appendix~\ref{sec_proof_lemma1}.
\end{proof}

Lemma 1 quantifies the gain of using stabilization when $\mathbf{w}^*$ is highly sparse, where stability selection efficiently shrinks high variable estimates to zero. When the purging threshold $\pi_0$ is sufficiently high such that $\pi_0 \in \left( \frac{1}{2} + \frac{(q^{\mathbf{g}})^2}{2p(q^{\mathbf{g}} - d^*)} , 1\right)$, the lower bound achieved by \eqref{eq: lemma1} is guaranteed to be positive. Furthermore, this result also indicates that the expected difference between distances from the non-stabilized and stabilized weight vector to the sparse $ \mathbf{w}^*$ depends on the differences between the sizes of the temporary nonzero set of features before purging, $|\hat{S}_{\varepsilon}|$, and the size of the stable features after purging. In expectation, the stabilized weight vector is closer to the target sparse weight vector as the operation of purging efficiently reduces the size of stable features. This suggests a much faster convergence with stabilization. Lemma 1 also provides an insight on the benefit from using adaptive gravity with annealed rejection rate. At the beginning of the algorithm, the gap between the size of $|\hat{S}_{\varepsilon}|$ and the size of the set of stable features $\hat{|\Omega}|$ is large when aiming for extensive exploration of different sparse combination of features. Hence, the improvement brought by stabilization is more substantial during the early state of learning period. As the algorithm proceeds and the set of stable features becomes smaller and stabler, it dwindles the leeway that allows the aforementioned two sets to be different. Consequently, the proposed algorithm is gradually tuned toward the standard stochastic gradient descent algorithm to facilitate better convergence at the later period of the learning process.

%%%%%%%%%%%%%%%%%%%%%%%%%%%%%%%%%%%%%%%%%%%%%%%%%%%
%%%%%%%%%%%%%%%%%%%%%%%%%%%%%%%%%%%%%%%%%%%%%%%%%%%
%%%%%%%%%%%%%%%%%%%%%%%%%%%%%%%%%%%%%%%%%%%%%%%%%%%

\begin{lemma} \label{lemma2}

Let $\mathbf{w}_0$ be the weight vector at initialization. After the first burst,  let $\bar{\mathbf{w}}_1$ be the truncated weight vector using universal gravity $g_0K$ as in Algorithm~\ref{alg:ssgd_1} and let $\hat{\mathbf{w}}_1$ be the truncated weight vector with informative truncation as in Algorithm~\ref{alg:ssgd_2}. Then, under Assumption \ref{assump_3}, 
\begin{align*}
&\mathbb{E} \left( ||\bar{\mathbf{w}}_1 - \mathbf{w}^*||^2 - ||\hat{\mathbf{w}}_1 - \mathbf{w}^* ||^2  \right) 
\geq 2g_0 \sum\limits_{t=1}^{K} \norm{ \zeta_t \mathbf{w}^* }_1
\geq 0   \numberthis  \label{eq: lemma2}
\end{align*}
where $\zeta_{t,j} =  \mathbbm{1}(|x_{t,j}| = 0)$ for $j = 1, \dots, p$ and $\zeta_t = (\zeta_{t,1}, \dots, \zeta_{t,p})^T$.
\end{lemma}

\begin{proof}:
See Appendix~\ref{sec_proof_lemma2}.
\end{proof}

In Lemma 2, we compare the distances towards $\mathbf{w}^*$ from \textit{1)} the weight vector with uniform gravity and \textit{2)} the weight vector with informative truncation that depends on the number of zero entries occurred in a burst. Such a gap suggests the effectiveness of informative truncation on sparse data in which feature sparsity is highly heterogeneous. In the scenarios where very few nonzero entries appear in a burst, the informative truncation imposes gravity that is proportional to the information presented in a burst. It is a fairer treatment than uniform truncation and leads to a large improvement in expectation. When features are all considerably dense in a burst, the informative truncation is equivalent to the uniform truncation.

%%%%%%%%%%%%%%%%%%%%%%%%%%%%%%%%%%%%%%%%%%%%%%%%%%%
%%%%%%%%%%%%%%%%%%%%%%%%%%%%%%%%%%%%%%%%%%%%%%%%%%%
%%%%%%%%%%%%%%%%%%%%%%%%%%%%%%%%%%%%%%%%%%%%%%%%%%%
In short, Lemma 1 demonstrates the improvement in expected squared error due to stabilization on the weight vector. Lemma 2, on the other hand, quantifies the improvements in reduce truncation bias when implementing informative truncation on sparse features with heterogeneous sparsity levels. 

Given Lemma 1 and Lemma 2, we have the expected regret bound of the proposed Algorithm \ref{alg:stable_ssgd} in Theorem 1. 

\begin{thm}

Consider the updating rules for the weight vector in Algorithm 3. On an arbitrary path, with $\underline{\mathbf{w}}_0 = 0$ and $\eta > 0$, let $\left\{ \underline{\mathbf{w}}_t \right\}_{t=1}^{T} $ be the resulted weight vector and $\left\{ \underline{g}_{t} \right\}_{t=1}^{T}$ be the gravity values applied to the weight vectors generated by Algorithm~\ref{alg:stable_ssgd}, along with the base gravity parameters $\left\{ \underline{g}_{0,t} \right\}_{t=1}^{T}$. Set the purging threshold $\pi_0$ to be sufficiently large such that $\pi_0 \in \left( \frac{1}{2} + \frac{(q^{\mathbf{g}})^2}{2p(q^{\mathbf{g}} - d^*)} , 1\right)$. If Assumption 1, 2, and 3 hold, then there exists a sequence of $\varepsilon_t \in \left( 0, \frac{\pi_0 C(p - |\hat{\Omega}_t|)}{|\hat{\Omega}_t|} \right)$ at each stability selection with the set of stable features $\hat{\Omega}_t$ such that the expectation of the regret defined in \eqref{eq_regret} is bounded above by

\begin{align*}
& \mathbb{E} \left( \sum\limits_{t=1}^{T} \left[  L(\underline{\mathbf{w}}_t, z_t) + K  \underline{g}_{t}  ||\underline{\mathbf{w}}_t||_1   \right]   - \sum\limits_{t=1}^{T} \left[  L(\mathbf{w}^*, z_t) + K \underline{g}_{t}    ||\mathbf{w}^*||_1   \right]  \right) \\
\leq &  \frac{\eta A}{2- \eta A} \left(  \mathbb{E} \left[ \sum\limits_{t=1}^{T}  L(\mathbf{w}^*, z_t)  + K \underline{g}_{0,t}  (||\mathbf{w}^*||_1 - ||\underline{\mathbf{w}}_t||_1 ) \right] \right) +  \frac{1}{2 - \eta A} \left( \eta T B + \frac{1}{\eta} || \mathbf{w}^*||^2 \right)  \\
 & - \frac{1}{2 \eta  - \eta^2 A } \sum\limits_{t=1}^{T} \varepsilon_t^2 \mathbbm{1} \left(\frac{t}{K n_K} \in \mathbb{Z} \right) ( |\hat{S}_{\varepsilon_t, t}| - |\hat{\Omega}_t|) \numberthis \label{eq: eq_0}
\end{align*}
where $\hat{S}_{\varepsilon, t} = \{j: \mathbb{E} (|\hat{w}_{t,j} |) > \varepsilon_t \}$ and $\hat{\mathbf{w}}_t$ is the weight vector at time $t$ before stabilization. 
\end{thm}

\begin{proof}:
See Appendix~\ref{sec_proof_thm1}.
\end{proof}

In the result of Theorem 1, the first two parts of the right-hand-side of the expected regret bound \eqref{eq: eq_0} is similar to the bound obtained in \cite{langford2009sparse}. It implies the trade-off between attained sparsity in the resulted weight vector and the regret performance. When the applied gravity is small under the joint effect of the base gravity $g_{0}$ and the size of each burst $K$, the sparsity is less but the expected regret bound is lower. On the other hand, when the applied gravity is large, the resulted weight vector is more sparse but at the risk of higher regret. Based on Lemma 1, the proposed algorithm is guaranteed to achieve lower regret bound in expectation when the target sparse weight vector is highly sparse. As quantified in the third term of the right-hand-side of  \eqref{eq: eq_0}, the improvement comes from the reduction of the active set at each purging. By its virtue, noisy features are removed from the set of stable features and thus are absent in later SGD updates and truncations. 

Theorem 1 is stated with a constant learning rate $\eta$. It is possible to obtain a lower regret bound in expectation with adaptive learning rate $\eta_t$ decaying with $t$, such as $\eta_t = \frac{1}{\sqrt{t}}$, which is commonly used in the literature of online learning and stochastic optimization. However, the discussion of using an varying learning rate is not a main focus of this paper and adds extra complexity of the analysis. Without knowing $T$ in advance, this may lead to a no-regret bound as suggested in \cite{langford2009sparse}. Instead, in Corollary 1, we show that the convergence rate of the proposed algorithm is $O(\sqrt{T})$ with $\eta = \frac{1}{\sqrt{T}}$. 

%%%%%%%%%%%%%%%%%%%%%%%%%%%%%%%%%%%%%%%%%%%%%%%%%
%%%%%%%%%%%%%%%%%%%%%%%%%%%%%%%%%%%%%%%%%%%%%%%%%

\begin{cor}
Assume that all conditions of Theorem 1 are satisfied. Let the learning rate $\eta$ be $\frac{1}{\sqrt{T}}$. The upper bound of the expected regret is 
$$
 \mathbb{E} \left( \sum\limits_{t=1}^{T} \left[  L(\underline{\mathbf{w}}_t, z_t) +  \underline{g}_t  ||\underline{\mathbf{w}}_t||_1   \right]   - \sum\limits_{t=1}^{T} \left[  L(\mathbf{w}^*, z_t) +  \underline{g}_t  ||\mathbf{w}^*||_1   \right]  \right) \leq O(\sqrt{T}),
 $$
 where $\underline{g}_t  = K \underline{g}_{0,t} $.
\end{cor}

\begin{proof}
By plugging in $\eta = \frac{1}{\sqrt{T}}$ to the result from Theorem 1, we get 
\begin{align*}
 & \mathbb{E} \left( \sum\limits_{t=1}^{T} \left[  L(\underline{\mathbf{w}}_t, z_t) +  \underline{g}_t  ||\underline{\mathbf{w}}_t||_1   \right]   - \sum\limits_{t=1}^{T} \left[  L(\mathbf{w}^*, z_t) +  \underline{g}_t  ||\mathbf{w}^*||_1   \right]  \right) \\
 \leq &  \frac{A}{2\sqrt{T} - A} \left(  \mathbb{E} \left[ \sum\limits_{t=1}^{T}  L(\mathbf{w}^*, z_t)  +  \underline{g}_t (||\mathbf{w}^*||_1 - ||\underline{\mathbf{w}}_t||_1  \right] \right) \\
 & +  \frac{T}{2\sqrt{T} - A } \left( \eta T B + \frac{1}{\eta} || \mathbf{w}^*||^2 \right) -  \frac{T}{2\sqrt{T} - A } \sum\limits_{t=1}^{T} \varepsilon_t^2 \mathbbm{1} (\frac{t}{K n_K} \in \mathbb{Z} )  ( |\hat{S}_{\varepsilon, t}| - |\hat{\Omega}_t|). \\
\end{align*}
The result is then straightforward. 
\end{proof}

Assume that the input features have $d$ nonzero entries on average. With linear prediction model $f_{\mathbf{w}}(x) = \mathbf{w}^T x$, the computational complexity at each iteration is $O(d)$. Leveraging the sparse structure, the informative truncation only requires an additional $O(Kd)$ space for recording the counters. The purging process of stability selection consumes $O(\delta)$, $\delta = K n_K M d$, space for storing the generated intermediate weight vectors and $O(\delta \log(\delta))$ computational complexity. Both storage and computational cost decrease when the set of stable features diminishes as the algorithm proceeds.  Since the parameters $K$, $n_K$, and $M$ is normally set to be small values, the complexity mostly depends on $O(d \log (d))$. In summary, the proposed algorithm scales with the number of nonzero entries instead of the total dimensions, making it appealing to high-dimensional applications.

\section{Practical Remarks} \label{sec_practical_remarks}

When implementing Algorithm~\ref{alg:stable_ssgd} in practice, the performance can be further improved in terms of both accuracy and computational efficiency by employing a couple of practical techniques. It includes applying {\em informative purging} and attenuating the truncation frequency to achieve more accurate sparse learning and steadier convergence.

The first improvement can be implemented by better addressing the issue of scarcity of incoming samples. For computing selection probabilities, instead of using only information from the current stage, we can inherit information from  previous stages for features that are too scarce to accumulate enough updates during one stage. Specifically, we introduce an \textit{accumulated counter} $\kappa_s$ at stage $s$ as the total number of times that a feature is updated within a burst during this stage:
$$
\kappa_{j,s} = \sum\limits_{m=1}^{M} \sum\limits_{\tau=1}^{n_K} \tilde{k}_{j,\tau}^{(m)}, \quad j=1, \dots, p,
$$
which is essentially the denominator of the selection probability in \eqref{eq_sel_prob_2}. Similarly, we define an \textit{accumulated truncation indicator} $\mathbf{b}_s$ at stage $s$ as the total number of times that a feature is truncated to zero given valid update(s):
$$
\mathbf{b}_{j,s} = \sum\limits_{m=1}^{M}  \sum_{\tau: \tilde{k}^{(m)}_{j,\tau} > 0}  \mathbbm{1}(|\hat{\mathbf{w}}^{(m)}_{j,\tau} | > 0), \quad j=1, \dots, p.
$$

A feature is then evaluated in the stability selection \textit{only} if there are enough updates from the present stage and from any unused information carried over from previous stages. Given a threshold $\delta_K \geq 0$, let $\tilde{\kappa}_{j,s} \triangleq  \tilde{\kappa}_{j,s-1} \mathbbm{1}(\tilde{\kappa}_{j,s-1} < \delta_K) + \kappa_{j,s}$ and $\tilde{\mathbf{b}}_{j,s} \triangleq \tilde{\mathbf{b}}_{j,s-1} \mathbbm{1}(\tilde{\kappa}_{j,s-1} < \delta_K) + \mathbf{b}_{j,s}$. The selection probability is modified as 
\begin{equation} \label{eq_sel_prob_3}
\hat{\Pi}_{j,s} = \left\{ \begin{array}{ll}
\frac{ \tilde{\mathbf{b}}_{j,s}}{\tilde{\kappa}_s }, & \mbox{for $j$ s.t.}  \tilde{\kappa}_s > \delta_{K} \\
1, & \mbox{otherwise}
\end{array}
\right., \quad \mbox{ for } j=1, \dots, p.
\end{equation}

This strategy extends the key idea in Section~\ref{sec_informative_truncation} that, with sparse data, each decision need to be based on sufficient evidence. Using the ``carried-over'' information allows the algorithm to utilize information available in a sequence of SGD updates while attuned to the needs of features with different levels of sparsity. In practice, this modification facilitates faster convergence especially for ultra-sparse data. 

%avail the algorithm of the scant information available in a sequence of SGD updates.

The second practical strategy is that the size of each burst, $K$, can be adaptively adjusted in a similar fashion as the rejection rate $\beta$ in \eqref{eq_rej_rate}. At the end of each stage, the burst size $K_s$ is updated as 
$$
K_s = \ceil[\Big]{K_0 \log \left( \frac{1}{\alpha d_{s-1}} \right)}, 
$$ 
where $K_0 > 0$ is the initial burst size and, as in \eqref{eq_rej_rate}, $d_s = \frac{| \hat{\Omega}_s|}{p}$. The tuning parameter $\alpha > 0$ adjusts the annealing rate of the truncation frequency. Although the result in Theorem 1 is based on a fixed $K$, it can be easily shown that the same upper bound can also be attained with an increasing $K_s$. By increasing $K$ in the later stage of the algorithm, when the majority of irrelevant features have been removed from the stable set, the chance of erroneous truncation is reduced. Such scheme further steers the algorithm from the mode of exploring potential sparse combination of features in the early stage toward the fine tuning of the weight vector by exploiting information from more samples in a sequence. It also facilitates faster convergence as the size of the stable set approaches to a sufficiently small number, as the algorithm converges to the standard stochastic gradient descent approximately.   

%[annealing frequency of truncations]
%
%$K_s = f(K_0, \frac{d_s}{p_s}, K_{\max}) $
%- $p_s$ = current size of $\hat{\Omega}_s$
%
%- $d_s$ = average column-wise sparsity of features in $\hat{\Omega}_s$

\section{Results}  \label{sec_emp_result}

In this section, we present experimental results evaluating the performance of the proposed stabilized truncated SGD algorithm in high-dimensional classification problems with sparsity regularization. In this paper, we focus on linear prediction model for binary classification where $f_{\mathbf{w}}(x) = \mathbf{w}^T x$ and $\hat{y} = \mbox{sign}(f_w(x))$ with the observed class label $y \in \{-1, 1\}$. We consider two commonly used convex loss functions in machine learning tasks that both satisfy Assumption~\ref{assump_1}:

\begin{itemize}
\item Hinge loss: $l(f, y) = \max(1-fy, 0)$
\item Logistic loss: $l(f, y) = \log\left(1+ \exp(-fy) \right)$
\end{itemize}

Using five datasets from different domains, the performance of our algorithm and other algorithms for comparison are evaluated on classification performance and feature selection stability and sparsity. We first define measure of feature stability in Section~\ref{sec:feature_stab}. 

%Then we extend the experiment to noncovex problem with image classification problem. 

\subsection{Feature Selection Stability} \label{sec:feature_stab}

The goal of sparse learning is to select a subset of truly informative features with stabilized estimation variance as well as increased classification accuracy and model interpretability. Subgradient-based online learning methods depend heavily by the random ordering of samples on which they are fed to the algorithm. Such dependence leads to much deteriorated performance when it comes to high-dimensional sparse inputs. For a particular feature, the positions of its nonzero occurrences in a random ordering of samples greatly affect its learning outcome, in terms of learnt weight and membership in the set of selected features. Therefore, in addition to attaining a low generalization error, a desirable sparse online learning method should also produce an informative feature subset that is stable and robust to random permutations of input data. To evaluate feature selection stability of subgradient-based sparse learning methods, we define in the following a numerical measure of similarity between selected feature subsets resulted from different random permutations of data. Given an output weight vector $\underline{\mathbf{w}}$ from a subgradient-based algorithm with input data $\mathcal{D}$, similarly as in \eqref{eq_selected_set}, we denote the selected feature subset as
$$
\mathcal{S}(\underline{\mathbf{w}}; \mathcal{D} )= \{ j: |\underline{w}_j| > 0, \underline{\mathbf{w}} = \Psi(\mathcal{D} \}.
$$

Given two random permutations of the training data $\mathcal{D}$, $ \mathcal{D}^{(1)}$ and $ \mathcal{D}^{(2)}$, the similarity between the two sets of selected feature subsets $\mathcal{S}_1 = \mathcal{S}(\underline{\mathbf{w}}_1; \mathcal{D}^{(1)} )$ and $\mathcal{S}_2 = \mathcal{S}(\underline{\mathbf{w}}_2; \mathcal{D}^{(2)} )$ is measured by the Cohen's kappa coefficient \citep{cohen1960kappa},
\begin{equation*}
\kappa(\mathcal{S}_1, \mathcal{S}_2) = \frac{q_o - q_e}{1- q_e},
\end{equation*}
where $q_o$ is the \textit{relative observed agreement} between $\mathcal{S}_1$ and $\mathcal{S}_2$:
$$
q_o = \frac{ p_{11}+p_{22}}{p},
$$
and $q_e$ is the \textit{hypothetical probability of change agreement}: $\mathcal{S}_1$ and $\mathcal{S}_2$
$$
q_e = \frac{(p_{11} + p_{12})(p_{11} + p_{21})}{p^2} + \frac{(p_{12} + p_{22})(p_{21} + p_{22})}{p^2}
$$
with $p_{11} = |\mathcal{S}_1 \cap \mathcal{S}_2|$, $p_{12}  = |\mathcal{S}_1 \cap \mathcal{S}^C_2|$, $p_{21} = |\mathcal{S}^C_1 \cap \mathcal{S}_2|$,  $p_{22}= |\mathcal{S}^C_1 \cap \mathcal{S}_2^C|$, and $p=p_{11} + p_{12} + p_{21} + p_{22}$ is the size of variable pool.

Note that $\kappa(\mathcal{S}_1, \mathcal{S}_2) \in [-1, 1]$, where $\kappa(\mathcal{S}_1, \mathcal{S}_2)=1$ if $\mathcal{S}_1$ and $\mathcal{S}_2$ completely overlap with each other and $\kappa(\mathcal{S}_1, \mathcal{S}_2)= -1$ when $\mathcal{S}_1$ and $\mathcal{S}_2$ are in complete disagreement with $\mathcal{S}_1 \cap \mathcal{S}_2 = \emptyset$ and $\mathcal{S}^C_1 \cap \mathcal{S}_2^C = \emptyset$.

Based on Cohen's kappa coefficient, we define the measure of feature selection stability of $\mathcal{S}(\underline{\mathbf{w}}; \cdot)$ returned by a procedure $\Psi$ using randomly ordered data $\mathcal{D}^{(1)}$ and $\mathcal{D}^{(2)}$ as
\begin{equation*}
s(\Psi) = \mathbb{E}_{\mathcal{D}^{(1)}, \mathcal{D}^{(2)}} \left[ \kappa(\mathcal{S}(\underline{\mathbf{w}}_1; \mathcal{D}^{(1)} ), \mathcal{S}(\underline{\mathbf{w}}_2; \mathcal{D}^{(2)} ) \right],
\end{equation*}
which is motivated by \cite{sun2013consistent}.

In practice, we use the empirical average $\hat{s}(\Psi)$ over $B$ random permutations of the training data to measure the stability of the a subgradient-based online learning algorithm $\Psi(\cdot)$:
\begin{equation} \label{eq_s_stab}
\hat{s}(\Psi) = \frac{1}{B(B-1)}\sum\limits_{i=1}^{D} \sum\limits_{j \neq i} \left[ \kappa(\mathcal{S}\underline{\mathbf{w}}_i; \mathcal{D}^{(i)} ), \mathcal{S}(\underline{\mathbf{w}}_j; \mathcal{D}^{(j)} ) \right].
\end{equation}

\subsection{Experiment Setup}

We evaluate the performance of our algorithm on several real-world classification datasets with up to $100,000$ features. These datasets have different levels of sparsity with various sample sizes. The information of experiment datasets are summarized in Table~\ref{tbl:data_summary}. The first four datasets were constructed for NIPS 2003 Feature Selection Challenge\footnote{Data source: \url{https://archive.ics.uci.edu/ml/index.html}} \citep{guyon2004result}, which were preprocessed with added ``probes" as random features distributed similarly to the real features. Thus, a good performance does not only lie in low generalization error rates, but also in sparse weight vectors that identify the truly important features. Reuters CV1 (RCV1) is a popular text classification dataset with a bag-of-words representation. We use the binary version from the LIBSVM dataset collection\footnote{Data source:   \url{https://www.csie.ntu.edu.tw/~cjlin/libsvmtools/datasets/}} introduced in \cite{cai2012manifold}. We create the training and validation set using a 70-30 random splits. All datasets are normalized such that each feature has variance 1. 

%Thus, in addition to the original sparsity inherent within the true data, a successful algorithm on these crafted datasets should not only achieve low generalization error but also induce considerable sparsity in the resulted weight vector that discriminates between truly significant features and the probes.

\begin{table}[H]
\caption{Datasets used in the numerical experiments. Sparsity is defined as the average feature sparsity levels, which is the column-wise average percentage of nonzero entries, in the training data.}
 \label{tbl:data_summary}
\vskip 0.15in
\begin{center}
\begin{small}
%\begin{sc}
\begin{tabular}{|c | c | r | R{1.9cm} | R{1.3cm} | R{1.5cm} |}
\hline
Dataset & Domain & Dimensions & Data Density (Sparsity) & Training Size & Validation Size \\
\hline
RCV1 & Text Mining & 47,236 & 0.16\% & 14,169 & 6,073 \\
Dexter & Text Mining & 20,000 & 0.48\% & 300 & 300 \\
Dorothea & Drug Discovery & 100,000 & 0.91\% & 800 & 350 \\
Gisette & Digits Recognition & 5,000 & 13.00\% & 6,000 & 1,000 \\
Arcene & Mass-Spectrometry & 10,000 & 50.00\% & 100 & 100 \\
\hline
\end{tabular}
%\end{sc}
\end{small}
\end{center}
\vskip -0.1in
\end{table}

We compare the proposed algorithm with the standard stochastic gradient descent algorithm \cite{bottou1998online} and another three other sparsity-inducing stochastic methods, including the truncated gradient algorithm \citep{langford2009sparse}, the Regularized Dual Averaging (RDA) algorithm \citep{xiao2009dual}, and the forward backward splitting (FOBOS) algorithm \citep{duchi2009efficient}. The RDA algorithm updates the weight vector at each step based on a running average  $\bar{\mathbf{g}}_t$ of all subgradients $\{\mathbf{g}_\tau = L' \left( \mathbf{w}_{\tau}, z_{\tau} \right) \in \partial_{\mathbf{w}_{\tau}} L \left( \mathbf{w}_{\tau}, z_{\tau} \right), \tau = 1, \dots, t  \}$ in previous iterations as
$$
\bar{\mathbf{g}}_t = \frac{t-1}{t} \mathbf{g}_{t-1} + \frac{1}{t}\mathbf{g}_t
$$
Given the average subgradient , the next weight vector is computed by solving the minimization problem
\begin{align} \label{eq_rda_opt}
\mathbf{w}_{t+1} = \underset{\mathbf{w}}{\operatorname{\mbox{arg}\min}} \left\{ \frac{1}{t} \sum\limits_{\tau}^{t} \langle L'(\mathbf{w}_{t}, z_t), \mathbf{w} \rangle + \Psi(\mathbf{w}) + \frac{\beta_t}{t} h(\mathbf{w}) \right\},
\end{align}
where $\Psi(\mathbf{w})$ is the regularizer,  $h(w)$ is an auxiliary strongly convex function, and $\{\beta_t\}_{t \geq 1}$ is a nonnegative and nondecreasing
input sequence, which determines the convergence properties of the algorithm.

In the context of $L_1$ regularization, the RDA algorithm is derived by setting $\Psi(\mathbf{w}) = \lambda ||\mathbf{w}||_1$, $\beta_t = \gamma \sqrt{t}$, and replacing $h(w)$ with a parametrized version:
$$
h_{\rho} = \frac{1}{2} ||\mathbf{w}||_2^2 + \rho ||\mathbf{w}||_1,
$$
where $\rho \geq 0$ is a sparsity-enhancing parameter. Hence, the minimization problem in \eqref{eq_rda_opt} has a explicit solution as, for $j = 1, \dots, p$,
$$
\mathbf{w}_{t+1}^{(j)} = \left\{
\begin{array}{ll}
0, & \mbox{if } |\bar{\mathbf{g}}_{t}^{(j)} | \leq \lambda_t^{\mbox{RDA}}\\
-\frac{\sqrt{t}}{\gamma} \left( \bar{\mathbf{g}}_{t}^{(j)} - \lambda_t^{\mbox{RDA}} \mbox{sgn} (\bar{\mathbf{g}}_{t}^{(j)}) \right), & \mbox{otherwise,}
\end{array}
\right. 
$$
which is equivalent to 
$$
\mathbf{w}_{t+1} = T(\bar{\mathbf{g}}_{t}, \lambda_t^{\mbox{RDA}}),
$$
where $\lambda_t^{\mbox{RDA}} = \lambda + \gamma \rho / \sqrt{t}$.

The FOBOS algorithm alternates between two phases. On each iteration, it first perform an unconstrained gradient descent step as in the standard SGD algorithm, whose output is denoted as $\mathbf{w}_{t + \frac{1}{2}}$. Then it cast and solve an instantaneous optimization problems that trades off between minimization of a regularization term and a close proximity to the result of the first phase:
\begin{equation} \label{eq_fobos_opt}
\mathbf{w}_{t+1} = \underset{\mathbf{w}}{\operatorname{\mbox{arg}\min}} \left\{ \frac{1}{2} ||\mathbf{w} - \mathbf{w}_{t + \frac{1}{2}} ||^2 + \eta_{t + \frac{1}{2}}\Psi(\mathbf{w})\right\},
\end{equation}
where the regularization function is scaled by an interim step size $\eta_{t + \frac{1}{2}}$.

With $L_1$ regularization where $\Psi(w) = \lambda ||\mathbf{w}||_1$, the second-phase update can be computed as 
$$
\mathbf{w}_{t+1} = T(\mathbf{w}_{t+\frac{1}{2}},  \eta_{t + \frac{1}{2}} \lambda).
$$

To evaluate the stability of resulted weight vectors, we randomly permute the indices of the training samples for $B=50$ times to produce stochastic samples that are fed to the algorithms. Such randomization helps identify the instability in learning results in terms of both error rate and the selected features. 

For implementing all five stochastic methods, we allow the algorithms to run on the training data for $\{5, 10, 20, 30, 40, 50, 60\}$ passes. In the standard stochastic gradient descent algorithm, we choose the optimal learning rate between 0.1 to 0.5. The base gravity parameter in the truncated gradient algorithm is chosen from the range $[0.001, 0.01]$. For RDA, we follow the suggestions in \cite{xiao2009dual} which sets $\gamma = 5000$ and $\rho = 0.005$ (effectively $\gamma \rho = 25)$. We tune the parameter $\lambda$ from the set of values $\{5\times10^{-5}, 1\times10^{-4},  5\times10^{-4},  1\times10^{-3}, 0.01, 0.05, 0.1, 0.5, 1, 5, 10 \}$ for both RDA and FOBOS. As instructed in \cite{duchi2009efficient}, $\eta_t$ is set to be $\frac{1}{\sqrt{t}}$ in FOBOS. For the proposed stabilized truncated SGD algorithm, the size of each burst (the truncation frequency) is fixed at $K=5$ and $n_K = 5$ for all datasets. In the proposed algorithm, we initiate the rejection rate at $\beta_0 = 0.7$ with annealing rate chosen from $\{-7, -5, -3, -1, 0, 1, 3\}$. The stability selection threshold $\pi_0$ is tuned within the range $[0.5, 0.6, \dots, 0.9]$. For multi-thread implementation in Algorithm~\ref{alg:stable_ssgd}, $M=16$ is used running in parallel on a high-performance computing cluster. 

\subsection{Results}

\subsubsection{Classification Performance}

As shown in Table~\ref{tbl:test_error_summary}, the proposed algorithm shows improvement over the truncated gradient algorithm over all datasets and has better performances than RDA and FOBOS in most of the experiments. As data density increases, the truncated gradient algorithm performs better with hinge loss than with logistic loss. With hinge loss, the algorithm updates the weight vector only for samples within a small margin from the boundary. With logistic loss, the algorithm updates with continuous increments for all incoming samples and favors dense features in sparse learning. With highly sparse samples sequentially feed to the algorithm, the truncation in every $K$ iterations is conducted with insufficient information about the true gradient due to the lack of nonzero entries in sparse features. Truncation with logistic loss leads to over selection of dense features and overfitting. The other two sparsity-inducing methods for comparison, RDA and FOBOS, also have large fluctuations across data sets with different levels of sparsity over these two loss functions, especially when data is highly sparse. In comparison, the performance of the proposed algorithm is consistent for various dimensions and sparsity levels. Our algorithm is also shown to be robust to different choices of loss function. The comparatively lower test errors, especially for highly sparse data, mainly owes to the proposed algorithm's fairer treatments of features with heterogeneous sparsity. 

From Table~\ref{tbl:test_error_summary} we also observe substantially lower variances in test errors under random permutations of samples. The proposed algorithm has the lowest standard deviations of test errors across all datasets and under both loss functions. Such an improvement over truncated gradient algorithm and other comparing methods comes from both informative truncation with adaptive gravity and stability selection. Based on these selection probabilities we carry out feature purging. Our selection probability is computed based on multiple bursts whose truncation is guided by fair amounts of shrinkage. Hence, the removal decisions of features from the set of stable variables are grounded in reliable information on feature importance, which is more likely to be shared across different permutations of the data. The accumulations of sufficient information for all features help the proposed algorithm to be robust to random fluctuations in online setting.

\begin{table}[h]
\centering
\caption{Mean test errors (\%) and the corresponding standard deviations with hinge loss and logistic loss over 50 random ordering of the training samples.}
\label{tbl:test_error_summary}
\vskip 0.15in
\begin{footnotesize}
\begin{tabular}{|c|c ||  R{1.2cm} |  R{1.3cm} | R{1.1cm} | R{1.1cm} |R{1.6cm} |}
\hline
    & Dataset  & \cellCenter{Standard SGD}  & Truncated Gradient & RDA & FOBOS & Stabilized Truncated SGD \\
\hhline{|=|=|=|=|=|=|=|}
& RCV1     & \textbf{2.86} & 6.59               & 4.62         & 8.44  & 3.01                     \\
 &  & \textit{.08} & \textit{.09} & \textit{.10} & \textit{.15}  & \textit{.04}                   \\
              & Dexter   & 7.77          & 8.37               & 11.42        & 10.91 & \textbf{6.58}            \\
              &          & \textit{.96} & \textit{.91}               & \textit{.59}         & \textit{1.21}  & \textit{.57}                      \\
Hinge Loss    & Dorothea & 6.11          & 6.83               & 6.83         & 8.51  & \textbf{5.14}            \\
              &          & \textit{.60}            & \textit{1.23}                & \textit{.18}          & \textit{1.96}   & \textit{0.43}                      \\
              & Gisette  & 2.71          & 6.32               & \textbf{2.40} & 5.50   & 3.05                     \\
              &          & \textit{.43}           & \textit{3.02}                & \textit{.45}          & \textit{1.20}    & \textit{.30}                      \\
              & Arcene   & 19.26         & 22.06              & 22.10         & 18.72 & \textbf{17.60}            \\
              &          & \textit{3.67}           & \textit{8.08}                & \textit{4.00}          & \textit{3.26}   & \textit{2.62}                     \\
\hhline{|=|=|=|=|=|=|=|}
              & RCV1     & 3.68          & 18.55              & 5.69         & 14.18 & \textbf{3.19}            \\
              &          & \textit{.07}          & \textit{.78}               & \textit{.08}         & \textit{.39}  & \textit{.05}                     \\
              & Dexter   & 7.91          & 10.25              & 9.41         & 10.45 & \textbf{6.41}            \\
              &          & \textit{.36}          & \textit{.58}               & \textit{.68}          & \textit{.81}   & \textit{.43}                      \\
Logistic Loss & Dorothea & 6.26          & 7.29               & 6.87         & 7.15  & \textbf{5.47}            \\
              &          & \textit{.62}           & \textit{1.28}                & \textit{.31}          & \textit{1.35}   & \textit{.35}                      \\
              & Gisette  & 2.52          & 6.35               & 2.21         & 5.98  & \textbf{2.11}            \\
              &          & \textit{.29}           & \textit{4.08}                & \textit{.20}           & \textit{.85}   & \textit{.26}                      \\
              & Arcene   & 17.38         & 19.48              & 21.62        & 18.44 & \textbf{13.38}           \\
              &          & \textit{2.93}          & \textit{5.01}               & \textit{2.05}          & \textit{3.11}   & \textit{1.26}           \\         
              \hline   
\end{tabular}%
\end{footnotesize}
\vskip -0.1in
\end{table}

\subsubsection{Feature Selection and Sparsity}

\begin{table}[h]
\centering
\caption{The average percentages (\%) of nonzero features selected in the resulted weight vector and the corresponding standard deviations with hinge loss and logistic loss over 50 random permutations of the training samples.}
 \label{tbl:sparsity_summary}
\vskip 0.15in
\begin{footnotesize}
\begin{tabular}{|c|c ||  R{1.2cm} |  R{1.3cm} | R{1.1cm} | R{1.1cm} |R{1.6cm} |}
\hline
 & Dataset & Standard SGD & Truncated Gradient & RDA & FOBOS & Stabilized Truncated SGD \\
\hhline{|=|=|=|=|=|=|=|} 
 & RCV1 & 63.68 & 6.95 & 11.38 & 5.01 & 0.86 \\
 &  & \textit{.34} & \textit{.42} &\textit{ .10} & \textit{.28} & \textit{.10} \\
 & Dexter & 29.46 & 14.74 & 1.61 & 37.69 & 1.98 \\
 &  & \textit{.35} & \textit{.82} & \textit{.06} & \textit{2.43} & \textit{.15} \\
Hinge Loss & Dorothea & 39.84 & 18.62 & 0.55 & 14.89 & 6.68 \\
 &  & \textit{.62} & \textit{5.94} &\textit{ .02} & \textit{3.37} &\textit{ .09 }\\
 & Gisette & 90.88 & 63.8 & 17.45 & 61.81 & 3.84 \\
 &  & \textit{.21} & \textit{3.72} & \textit{.43} & \textit{3.94} & \textit{.34} \\
 & Arcene & 98.22 & 61.08 & 12.15 & 54.5 & 3.67 \\
 &  & \textit{.07} & \textit{25.55} & \textit{.90} & \textit{13.37} & \textit{.40} \\
 \hhline{|=|=|=|=|=|=|=|}
 & RCV1 & 82.31 & 2.75 & 6.98 & 2.89 & 1.44 \\
 &  & \textit{.00} & \textit{.35} & \textit{.05} & \textit{.29} &\textit{ .12} \\
 & Dexter & 38.76 & 13.91 & 4.84 & 30.21 & 1.32 \\
 &  & \textit{.00} & \textit{.71} & \textit{.14} & \textit{1.87} & \textit{.13} \\
Logistic Loss & Dorothea & 68.63 & 19.31 & 0.43 & 20.4 & 1.67 \\
 &  & \textit{.00} & \textit{3.00} & \textit{.01} & \textit{3.51} &\textit{ .47} \\
 & Gisette & 93.96 & 66.31 & 14.99 & 58.83 & 3.22 \\
 &  & \textit{.00} & \textit{3.31} & \textit{.34 }& \textit{3.44} &\textit{ .49} \\
 & Arcene & 98.36 & 88.34 & 12.71 & 45.96 & 5.35 \\
 &  & \textit{.00} & \textit{4.38} & \textit{.27} & \textit{9.76} & \textit{.24} \\
 \hline
\end{tabular}%
\end{footnotesize}
\vskip -0.1in
\end{table}

As far as sparse learning for feature selection goes, the proposed algorithm, as shown in Table~\ref{tbl:sparsity_summary}, is the only method that delivers sparse weight vectors across five example datasets with various sparsity levels. It is shown in Table \ref{tbl:sparsity_summary} that the resulted weight vectors are far more sparse than the truncated gradient algorithm in all datasets with both loss functions. When comparing these two algorithms, the distinction in attained sparsity levels is particularly striking with logistic loss function and when the sparsity level of the input data is low, such as the case of the Arcene dataset. These results also indicate that the truncate gradient algorithm fails to extract significant features and to obtain sparse solutions, which motivates the development of techniques discussed in this paper. On the other hand, the variances in the percentage of nonzero features of the proposed algorithm are reduced by approximately a magnitude of 10. The contribution of stabilization is demonstrated again in terms of feature selection.  Although RDA achieves very low sparsity in Dorothea, such a behavior is not observed in other datasets. It is shown to result in particularly denser weight vector with highly sparse data, such as RCV1, indicating its weakness in identifying truly informative features when information is scarce. FOBOS demonstrates overall poorer performance in terms of inducing sparse weight vector as compared to RDA and the proposed algorithm. Similar to its performance in terms of test error, the proposed algorithm delivers highly stable results in feature selection regardless of the choice of loss function and of random permutations of the data. We also achieve sufficient sparse result owing to stability selection which prevent noisy features from adding back to the stable set of variables in the online setting.  On the other hand, the high sparsity in weight vector does not overshadow the generalizability performance as the informative truncation with unbiased shrinkage underlies a better estimation of the selection probability for the construction of the set of stable variables.

We further compare the data sparsity levels of the features selected by both sparse online learning algorithms. In Figure~\ref{fig: spar_ratio_dorothea}, using the Dorothea dataset as an example, we show the fraction of selected features at different data density levels. It is shown that truncated gradient algorithm is more likely to select dense features over sparse ones with both loss functions. The majority of the sparse features are truncated to zero regardless of their importance. Moreover, the fraction of selected feature exhibits a linear relationship with the level of feature density in truncated gradient algorithm. This pattern suggests that the amount of shrinkage applied to features should be approximately proportional to their data density, on the probability of having informative updates. This provides an independent justification of the informative truncation introduced in Section~\ref{sec_informative_truncation}. In contrast, features selected by the proposed algorithm have approximately uniform distribution of sparsity levels. This improvement over the truncated gradient algorithm mostly owes to the use of informative truncation. With the crucial treatment on the heterogeneity in feature sparsity, the proposed algorithm is effective in keeping rare but important features from premature truncations.

\begin{figure} 
\centering
\includegraphics[width=12cm]{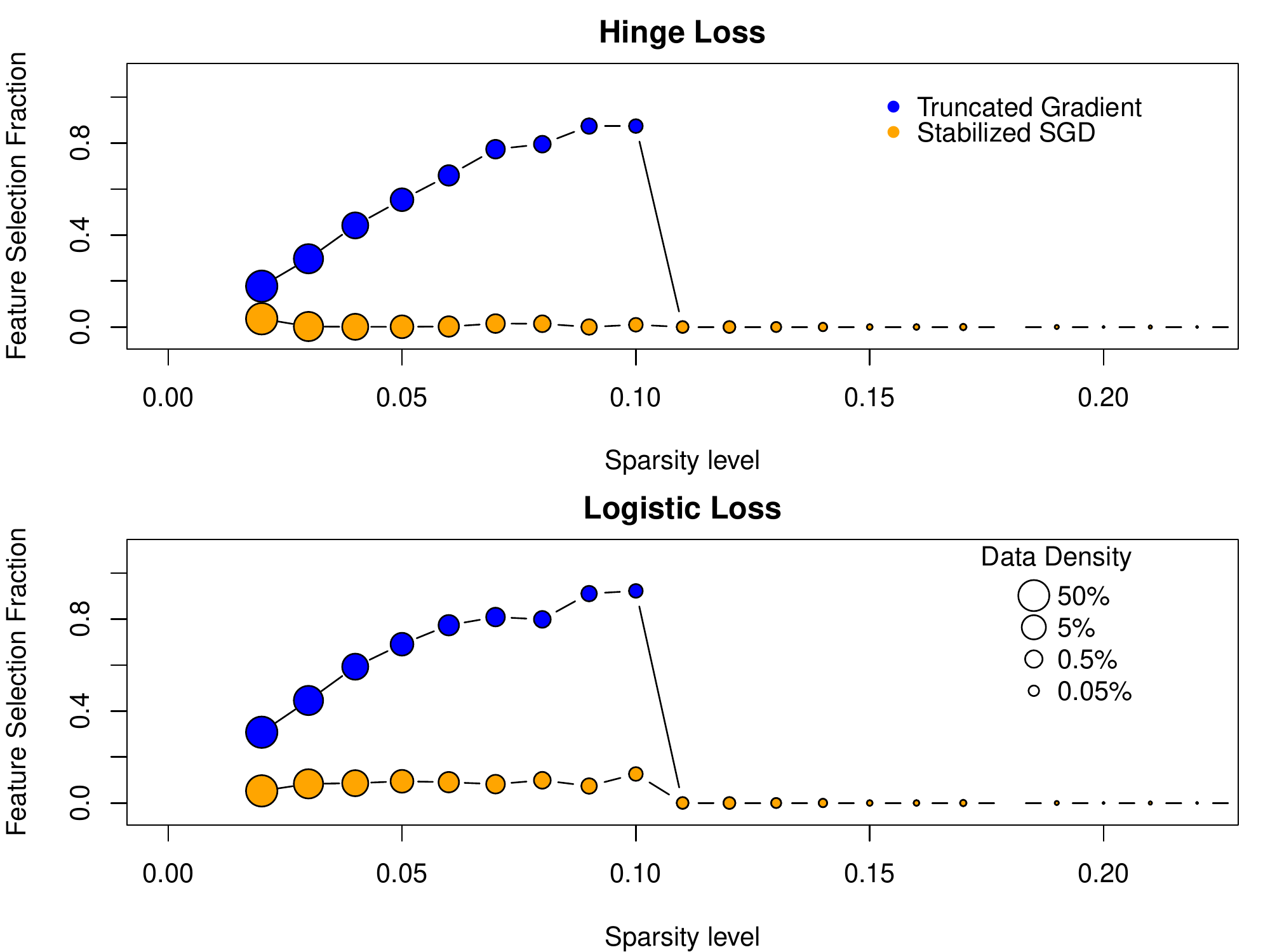}
\caption{The fraction of features selection at different sparsity levels by the truncated gradient algorithm and the proposed algorithm respectively with the Dorothea dataset as an example. The selected features are extracted as the features with nonzero entries in the resulted weight vector from the algorithms with a single fixed ordering of the training data. The $x$-axis represents the sparsity level, which is the proportion of nonzero entries in a feature in the training sample. The $y$-axis indicates the fraction of selected features among features with a discretized sparsity level. The size of each point scales with the proportion of features that fall in level of sparsity. It can be seen that the truncated gradient algorithm over-selects denser features. The proposed algorithm does not have a pre-exposed preference towards any density level.}
\label{fig: spar_ratio_dorothea}
\end{figure}

Based on the measure of feature selection stability defined in \eqref{eq_s_stab}, the performance of the original truncated gradient algorithm \citep{langford2009sparse} and the proposed algorithm are summarized in Table~\ref{tbl:stability_summary}. The proposed algorithm demonstrates its excellent stability as compared to the truncated gradient algorithm. It shows greater overlap in selected features among different permutations of the training data. This result also echos with the lower standard deviations in the test error (Table \ref{tbl:test_error_summary}) and in the number of selected features (Table \ref{tbl:sparsity_summary}). As discussed in Section~\ref{sec_stabilized_truncated_SGD}, the insensitivity to data perturbation is mainly due to the fair shrinkage applied on each feature by using informative truncation and to the sufficient accumulations of information in selection probabilities before applying stability selection. These measures suggest that the proposed algorithm is particularly favorable for high-dimensional sparse data.

\begin{table}[h]
\caption{Summary of feature selection stability performance of the resulted weight vector measured by \eqref{eq_s_stab} based on Cohen's kappa coefficient.}
 \label{tbl:stability_summary}
\vskip 0.15in
\begin{center}
\begin{small}
\begin{tabular}{| c ||  C{1.7cm} | C{2cm} ||   C{1.7cm} | C{2cm} |}
\hline
 & \multicolumn{2}{|c||}{Hinge Loss} & \multicolumn{2}{|c|}{Logistic Loss} \\
 \hline 
Dataset &  \tabbox[b]{Truncated Gradient} &  Stabilized Truncated SGD & \tabbox[b]{Truncated Gradient} &  Stabilized Truncated SGD \\
\hline
RCV1& 0.38 & 0.44 & 0.29  &  0.60 \\
Dexter & 0.46  &  0.61 & 0.23 &  0.58\\
Dorothea & 0.19  &  0.52  & 0.19 &    0.33\\
Gisette & 0.48 &   0.60 & 0.54  &  0.54 \\
Arcene & 0.26  & 0.55 & 0.35  & 0.69  \\
\hline
\end{tabular}
\end{small}
\end{center}
\vskip -0.1in
\end{table}

\section{Conclusion} \label{sec_conclusion}
In this paper, we address the problem of inducing sparsity in subgradient-based online learning algorithm when applied to high-dimensional sparse data. Based on the truncated gradient framework, we reduce truncation bias due to heterogeneity in feature sparsity via informative truncation. Integrated with soft-thresholding truncation, stability selection helps eliminate irrelevant features along the learning process in order to achieve stable and sufficiently sparse results. The adaptive gravity method dynamically adjusts the shrinkage at different stages of the learning process to balance between the exploration of sparse combination of nonzero weights and the exploitation of finite updates for better convergence. This strategy also consolidates the tuning needs of the proposed algorithm into two major parameters, the annealing rate $\gamma$ and the purging threshold $\pi_0$. We present a theoretical analysis of the expected regret bound, which offered a quantification of the potential improvement attained by the proposed algorithm. At last, we use extensive experimental studies to evaluate the performance of the proposed algorithm. The results demonstrate that our algorithm achieve favorable results in terms of both prediction accuracy and feature selection, compared to the original truncated gradient algorithm.

\newpage
\bibliographystyle{abbrvnat}
\bibliography{ssgd_bib.bib}

\begin{appendices}
\section{Proof of Lemma~\ref{lemma1}} \label{sec_proof_lemma1}
\begin{proof}
%\textit{Proof:}
Note that $\mathbf{w}^*$ can be written as $ \mathbf{w}^* \mathbbm{1}_{\Omega^*}$. Based on the stability selection in \eqref{eq_stability_selection},  we have $\tilde{\mathbf{w}} = \hat{\mathbf{w}} \mathbbm{1}_{\hat{\Omega}}$.  Let $\Omega = \{1, \dots, p\}$. The full set $\Omega$ can be further divided into four disjoint sets: $S_1 = \Omega \setminus (\Omega^* \cup \hat{\Omega})$, $S_2 =  \Omega^* \setminus \hat{\Omega}$, $S_3 =  \hat{\Omega} \setminus \Omega^*$, and $S_4 = \Omega^* \cap \hat{\Omega}$, respectively. 

Firstly, we consider the $L_2$ distance from the non-stabilized weight vector $\hat{\mathbf{w}}$ to $\mathbf{w}^*$ under Assumption~\ref{assump_3} and the $L_2$ distance from the stabilized weight vector $\tilde{\mathbf{w}}$ to $\mathbf{w}^*$.
\begin{align*}
||\hat{\mathbf{w}} - \mathbf{w}^*||^2 & = ||\hat{\mathbf{w}} - \mathbf{w}^*  \mathbbm{1}_{\Omega^*} ||^2 \\
& = \sum_{j \in \Omega^* } (\hat{w}_j - w_j^*)^2 + \sum_{j \notin \Omega^*} \hat{w}_j^2,\\
& = \sum_{j \in S_2 \cup S_4 } (\hat{w}_j - w_j^*)^2 + \sum_{j  \in S_1 \cup S_3} \hat{w}_j^2.
\end{align*}

Since
\begin{equation*}
(\hat{\mathbf{w}} \mathbbm{1}_{\hat{\Omega}} - \mathbf{w}^* \mathbbm{1}_{\Omega^*})^2_j  = \left\{ 
\begin{array}{ll}
0,& \mbox{if } j \notin \Omega^* \mbox{ and } j \notin \hat{\Omega}; \\
(w_j^*)^2,& \mbox{if } j \in \Omega^* \mbox{ and } j \notin \hat{\Omega} ;\\
(\hat{w}_j)^2,& \mbox{if } j \notin \Omega^* \mbox{ and } j \in \hat{\Omega}; \\
(\hat{w}_j - w_j^*)^2,& \mbox{if } j \in \Omega^* \mbox{ and } j \in \hat{\Omega}, \\
\end{array}
\right.
\end{equation*}

\begin{align*}
|| \tilde{\mathbf{w}} - \mathbf{w}^* ||^2 & = ||\hat{\mathbf{w}} \mathbbm{1}_{\hat{\Omega}} - \mathbf{w}^* \mathbbm{1}_{\Omega^*} ||^2 \\
& = \sum\limits_{j \in S_2} (w_j^*)^2 + \sum\limits_{j \in S_3} (\hat{w}_j)^2 + \sum\limits_{j \in S_4}  (\hat{w}_j - w_j^*)^2.
\end{align*}
Then
\vspace{-0.5cm}
\begin{align*}
||\hat{\mathbf{w}} - \mathbf{w}^*||^2 - ||\tilde{\mathbf{w}} - \mathbf{w}^* ||^2 & = \sum\limits_{j \in S_2}  (\hat{w}_j - w_j^*)^2 +  \sum\limits_{j \in S_1} \hat{w}_j^2 -  \sum\limits_{j \in S_2} (w_j^*)^2 \\
& = \sum\limits_{j \in S_2} \hat{w}_j^2 + \sum\limits_{j \in S_2}  (w_j^*)^2 - 2 \sum\limits_{j \in S_2}  w_j^* \hat{w}_j - \sum\limits_{j \in S_2}  (w_j^*)^2 + \sum\limits_{j \in S_1} \hat{w}_j^2 \\ 
& = \sum\limits_{j \in S_1 \cup S_2} \hat{w}_j^2 - 2 \sum\limits_{j \in S_2} w_j^* \hat{w}_j \\
& = || \hat{\mathbf{w}} ( 1 - \mathbbm{1}_{\hat{\Omega}} ) ||^2 - 2 \sum\limits_{j \in S_2}  w_j^* \hat{w}_j \\ \numberthis \label{eq_eq_00}
%& \geq \hat{\mathbf{w}}_{\min}^2 ( p - |\hat{\Omega} | ) - 2 \sum\limits_{j \in S_2}  w_j^* \hat{w}_j.\\
\end{align*}

For the first part in \eqref{eq_eq_00}, there exists a small $\varepsilon \in \left( 0, \frac{\pi_0 C(p - |\hat{\Omega}|)}{|\hat{\Omega}|} \right)$ with $\hat{S}_{\varepsilon} = \{j: \mathbb{E} (|\hat{w}_j |) > \varepsilon \}$ such that 
% there exists some small $\varepsilon > 0$ with $\hat{S}_{\varepsilon} = \{j: |\hat{w}_j | > \varepsilon \}$ such that $\sum\limits_{j \notin \hat{\Omega}} \mathbbm{1} ( |\hat{w}_j | > \varepsilon) > |\hat{\Omega} |$,
\vspace{-0.5cm}
\begin{align*}
\mathbb{E} \left(|| \hat{\mathbf{w}} (1-  \mathbbm{1}_{\hat{\Omega}} ) ||^2 \right)  = & \mathbb{E} \left(\sum\limits_{j=1}^{p} \hat{w}_j^2 \mathbbm{1}(\hat{p}_j \leq \pi_0) \right)\\
= &\mathbb{E} \left( \sum\limits_{j=1}^{p} \hat{w}_j^2  \mathbbm{1}(\hat{p}_j \leq \pi_0) \mathbbm{1}( |\hat{w}_j|  \leq \varepsilon ) \right)+\mathbb{E} \left( \sum\limits_{j=1}^{p} \hat{w}_j^2  \mathbbm{1}(\hat{p}_j \leq \pi_0) \mathbbm{1}( |\hat{w}_j|  > \varepsilon ) \right)\\
\geq & 0 + \varepsilon^2 \mathbb{E} \left( \sum\limits_{j=1}^{p}  \mathbbm{1}(\hat{p}_j \leq \pi_0) \mathbbm{1}( |\hat{w}_j|  \leq \varepsilon ) \right) \\
\geq & \varepsilon^2 ( |\hat{S}_{\varepsilon}| - | \hat{\Omega} | ) >  0, \numberthis \label{eq_lemma1_eq0b}
\end{align*}
where the last inequality is due to $|\hat{S}_{\varepsilon}| > | \hat{\Omega} |$ with the specified range of $\varepsilon$.

For the second part in \eqref{eq_eq_00}, from on Theorem 1 of \cite{meinshausen2010stability} in stability selection, we have
$$
\mathbb{E} ( | S_2| )  \leq  \left( \frac{q^g}{2\pi_0 p - p} - 1 \right) q^g + d^*.
$$
And, since, for any $j \in \Omega \setminus \hat{\Omega}$, $\Pr(|\hat{w}_j| > 0) \leq \pi_0$, 
$$
\mathbb{E}( |\hat{w}_j|) =  \int_{0}^{C} \Pr(|\hat{w}_j| \geq u) du \leq  \Pr(|\hat{w}_j| > 0) C  \leq \pi_0 C.
$$
Then,
\vspace{-0.5cm}
\begin{align*}
\mathbb{E} \left(  \sum\limits_{j \in S_2} w_j^* \hat{w}_j \right) \leq \pi_0 C^2 \mathbb{E} (|S_2|) \leq \pi_0 C^2  \mathbb{E} (|S_2|)  \left[\left( \frac{q^g}{2\pi_0 p - p} - 1 \right) q^g + d^* \right].
\end{align*}
Thus we have 
\vspace{-0.3cm}
\begin{align} \label{eq_lemma1_eq0c}
\mathbb{E} \left( ||\hat{\mathbf{w}} - \mathbf{w}^*||^2 - ||\tilde{\mathbf{w}} - \mathbf{w}^* ||^2  \right) \geq \varepsilon^2 ( |\hat{S}_{\varepsilon}| - |\hat{\Omega} | ) + 2 \pi_0 C^2 \left[  \left( 1 - \frac{q^g}{2\pi_0 p - p} \right) q^g - d^* \right] . 
\end{align}

The first half of \eqref{eq_lemma1_eq0c} is proved to be non-negative by the inequality in \eqref{eq_lemma1_eq0b}. With a sufficiently large $\pi_0 \in \left( \frac{1}{2} + \frac{(q^{\mathbf{g}})^2}{2p(q^{\mathbf{g}} - d^*)} , 1\right)$, \eqref{eq_lemma1_eq0c}  is non-negative by taking $\pi_0$ into the second half of it with $C \geq 0$ under Assumption 1.

\end{proof}

\section{Proof of Lemma 2} \label{sec_proof_lemma2}

\begin{proof}
Based on Algorithm~\ref{alg:ssgd_1} and Algorithm~\ref{alg:ssgd_2}, we have the following relationships:
\begin{align*} 
\bar{\mathbf{w}}_1 =  & T(\mathbf{h}_1, Kg_0) =  \mbox{sign}(\mathbf{h}_1) \max ( |\mathbf{h}_1| -g_{0} \mathbf{1}_p K, 0), \numberthis \label{eq_lemma2_eq1a}\\
\hat{\mathbf{w}}_1 = &  T(\mathbf{h}_1, \tilde{\mathbf{k }}g_0) = \mbox{sign}(\mathbf{h}_1) \max ( |\mathbf{h}_t| -g_{0} \tilde{\mathbf{k}}_{1}, 0),\numberthis \label{eq_lemma2_eq1b}
\end{align*}
where $$\mathbf{h}_1 = \mathbf{w}_0 - \sum\limits_{t=1}^{K}  \eta L' \left( \mathbf{w}_{t-1}, z_t \right). $$
Without loss of generality, we consider the difference in the squared errors to the optimal weight vector $\mathbf{w}^*$ between$\bar{\mathbf{w}}_1$ and $\hat{\mathbf{w}}_1$:
\begin{align*}
& ||\bar{\mathbf{w}}_1 - \mathbf{w}^*||^2 -   ||\hat{\mathbf{w}}_1 - \mathbf{w}^* ||^2 \\
= & || (\bar{\mathbf{w}}_1 - \mathbf{h}_1) - (\mathbf{w}^* -\mathbf{h}_1) ||^2 -  || (\hat{\mathbf{w}}_1 - \mathbf{h}_1) - (\mathbf{w}^* -\mathbf{h}_1) ||^2 \\
= & \left(  || \bar{\mathbf{w}}_1 - \mathbf{h}_1||^2 -  || \hat{\mathbf{w}}_1 - \mathbf{h}_1||^2 \right)  - 2\left[  ( \bar{\mathbf{w}}_1 - \mathbf{h}_1)^T(\mathbf{w}^* - \mathbf{h}_1) - ( \hat{\mathbf{w}}_1 - \mathbf{h}_1)^T(\mathbf{w}^* - \mathbf{h}_1)  \right] \numberthis  \label{eq_lemma2_eq2}
\end{align*}
In the first part of \eqref{eq_lemma2_eq2}, based on \eqref{eq_lemma2_eq1a} and \eqref{eq_lemma2_eq1b},  we have
\begin{align*}
|| \bar{\mathbf{w}}_1 - \mathbf{h}_1||^2 & = pK^2 g_0^2, \\
|| \hat{\mathbf{w}}_1 - \mathbf{h}_1||^2 & =  ||\tilde{\mathbf{k}}_1||^2 g_0^2, \\
\end{align*} 
where $\tilde{\mathbf{k}}_1$ is the counter of informative updates in the first burst with $\tilde{k}_{1,j} \leq K$ for $j = 1, \dots, p$.
\noindent Hence, 
 \begin{align} \label{eq_lemma2_eq2a}
 || \bar{\mathbf{w}}_1 - \mathbf{h}_1||^2 -  || \hat{\mathbf{w}}_1 - \mathbf{h}_1||^2  \geq 0 
 \end{align}
 Based on \eqref{eq_lemma2_eq1a} and \eqref{eq_lemma2_eq1b},
 \begin{align*}
  (\bar{\mathbf{w}}_1 - \mathbf{h}_t)^T \mathbf{h}_1  = & -g_{0} \sum\limits_{j=1}^{p} \mbox{sign}(h_{1,j}) K \mathbf{h}_{1,j}  = -g_{0} K|| \mathbf{h}_t ||_1, \\
 (\hat{\mathbf{w}}_1 - \mathbf{h}_t)^T\mathbf{h}_1  = & -g_{0} \sum\limits_{j=1}^{p} \mbox{sign}(h_{1,j}) \tilde{k}_{1,j} \mathbf{h}_{1,j}  = -g_{0} || \tilde{\mathbf{k}}_{1} \mathbf{h}_t ||_1.\\
 \end{align*}
 Thus in the second part of \eqref{eq_lemma2_eq2}, we have
 \begin{align*}
& - 2\left[  ( \bar{\mathbf{w}}_1 - \mathbf{h}_1)^T(\mathbf{w}^* - \mathbf{h}_1) - ( \hat{\mathbf{w}}_1 - \mathbf{h}_1)^T(\mathbf{w}^* - \mathbf{h}_1)  \right] \\
= & \left[  (\bar{\mathbf{w}}_1 - \mathbf{h}_1)^T \mathbf{h}_1 -  (\hat{\mathbf{w}}_1 - \mathbf{h}_1)^T \mathbf{h}_1 \right] - \left[ (\bar{\mathbf{w}}_1 - \mathbf{h}_1)^T\mathbf{w}^*- (\hat{\mathbf{w}}_1 - \mathbf{h}_1)^T\mathbf{w}^*  \right] \\
= &  2g_0 \left[ \left( K \norm{ \mathbf{h}_1}_1 - \norm{\tilde{\mathbf{k}}_1 \mathbf{h}_1}_1 \right) - (\bar{\mathbf{w}}_1 -\hat{\mathbf{w}}_1)^T  \mathbf{w}^* \right] \numberthis  \label{eq_lemma2_eq3}
 \end{align*}
Since $K \geq \tilde{k}_{1,j}$ for $j = 1, \dots, p$, the first part of \eqref{eq_lemma2_eq3} is guaranteed to be nonnegative. In the second part of  \eqref{eq_lemma2_eq3} , let $\hat{w}_{1,j} = \hat{v}_{1,j} \mathbbm{1}( \abs{h_{1,j}} > \tilde{k}_{1,j} g_0)$ and  $\bar{w}_{1,j} = \bar{v}_{1,j} \mathbbm{1}( \abs{h_{1,j}} >K g_0)$. Denote $\delta_t = L(w_{t-1}, z_t)$ to be the gradient of the loss function at a certain iteration. Since in this paper we consider linear prediction model, 
$$
\delta_t = L'(\mathbf{w}_{t-1}, z_t) = G(f_{\mathbf{w_{t-1}}}(\mathbf{x}_t))\mathbf{x}_t.
$$
Hence, $\delta_{t,j} = \delta_{t,j} \mathbbm{1}(|x_{t,j}| >0)$.

Again based on relationships in \eqref{eq_lemma2_eq1a} and \eqref{eq_lemma2_eq1b}, we consider the following two scenarios: 
\begin{itemize}
\item When $h_{1,j} > 0$:
\begin{align*}
\hat{v}_{1,j} = & \sum\limits_{t=1}^{K} \left[ \mathbbm{1}(|x_{t,j}| >0)( \delta_{t,j} - g_0) \right], \\
\bar{v}_{1,j} = & \sum\limits_{t=1}^{K} \left[  \mathbbm{1}(|x_{t,j}| >0)\delta_{t,j} - g_0 \right]. 
 \end{align*}
 Thus, 
\begin{align*}
 \bar{v}_{1,j} - \hat{v}_{1,j} = & - \sum\limits_{t=1}^{K} \mathbbm{1}(|x_{t,j}| =0) g_0, \\
 \bar{w}_{1,j} - \hat{w}_{1,j} = & - \sum\limits_{t=1}^{K} \mathbbm{1}(|x_{t,j}| =0) g_0 \mathbbm{1}(h_{1,j} > Kg_0) \\
 &  -  \sum\limits_{t=1}^{K} (h_{1,j} - \tilde{k}_{1,j} g_0) \mathbbm{1} (\tilde{k}_{1,j} g_0 < h_{1,j} < Kg_0)  \\
 \leq &  - \sum\limits_{t=1}^{K} \mathbbm{1}(|x_{t,j}| =0) g_0 \mathbbm{1}(h_{1,j} > Kg_0).\numberthis  \label{eq_lemma2_eq4}
\end{align*} 
 \item When $h_{1,j} < 0$, similarly, we get 
\begin{align*}
 \bar{v}_{1,j} - \hat{v}_{1,j} = &  \sum\limits_{t=1}^{K} \mathbbm{1}(|x_{t,j}| =0) g_0,\\
 \bar{w}_{1,j} - \hat{w}_{1,j} = &  \sum\limits_{t=1}^{K} \mathbbm{1}(|x_{t,j}| =0) g_0 \mathbbm{1}(h_{1,j} < -Kg_0)  \\
 & +  \sum\limits_{t=1}^{K} (h_{1,j} + \tilde{k}_{1,j} g_0) \mathbbm{1} (- Kg_0< h_{1,j} < -\tilde{k}_{1,j} g_0)  \\
 \geq &  \sum\limits_{t=1}^{K} \mathbbm{1}(|x_{t,j}| =0) g_0 \mathbbm{1}(h_{1,j} < -Kg_0). \numberthis  \label{eq_lemma2_eq5}
\end{align*} 
\end{itemize}
 Hence, based on  jointly we have 
\begin{align*}
& \mathbb{E} \left( ( \bar{\mathbf{w}}_1 -\hat{\mathbf{w}}_1)^T  \mathbf{w}^* \right)\\
= &  \sum\limits_{j=1}^{p} \mathbb{E}( (\bar{w}_{1,j} - \hat{w}_{1,j}) w_j^*| h_{1,j} > 0) \Pr(h_{1,j} > 0) + \mathbb{E}( (\bar{w}_{1,j} - \hat{w}_{1,j}) w_j^*| h_{1,j} < 0) \Pr(h_{1,j} < 0) \\
\leq &  -g_0 \sum\limits_{t=1}^{K}\sum\limits_{j=1}^{p} \mathbbm{1}(|x_{t,j}| = 0)|w^*_j| \leq 0
\end{align*}
The second inequality is derived from the condition that $\lambda_j = 0$ if $\mathbf{w}^*_j = 0$ and the following fact. Since $\mathbf{h}_1$ is the sum of $K$ stochastic gradients descent steps and each one of the stochastic gradient which is an unbiased estimator of the true gradient \cite{bottou1998online}, we have 
\begin{align*}
\mathbb{E}( \mathbbm{1}(h_{1,j} < 0) w_j^*) & < 0, \\
\mathbb{E}( \mathbbm{1}(h_{1,j} > 0) w_j^*) & > 0. 
\end{align*} 
Hence, together with \eqref{eq_lemma2_eq2a} and \eqref{eq_lemma2_eq3}, we have 
$$
 ||\bar{\mathbf{w}}_1 - \mathbf{w}^*||^2 -   ||\hat{\mathbf{w}}_1 - \mathbf{w}^* ||^2 \geq 0.
$$
\end{proof}

\vspace{-1cm}
\section{Proof of Theorem 1} \label{sec_proof_thm1}
\begin{proof}
At a given time $t$ when truncation is performed, let $\mathbf{h}_t =  \mathbf{w}_{t-1} - \eta \nabla_{\mathbf{w}_{t-1}} L(\mathbf{w}_{t-1}, z)$ and let $\hat{\mathbf{w}}_t$ be the truncated but non-stabilized weight vector based on the truncation in \eqref{eq_weight_trunc} that $\hat{\mathbf{w}}_t = T(\mathbf{h}_t, g_{0,t} \tilde{\mathbf{k}}_t)$ with the base gravity $g_{0,t}$. Let $\hat{\Omega}$ be the current set of stable features and the stabilized weight vector is obtained as $\tilde{\mathbf{w}}_t = \hat{\mathbf{w}}_t \mathbbm{1}_{\hat{\Omega}}$. 

Firstly, we have 
\begin{align*}
||\hat{\mathbf{w}}_t - \mathbf{w}^* ||^2 & \leq || \mathbf{w}^*  - \hat{\mathbf{w}}_t ||^2 + ||\hat{\mathbf{w}}_t - \mathbf{h}_t ||^2  \numberthis  \label{eq_eq_1} \\
& = || \mathbf{w}^* - \mathbf{h}_t ||^2 - 2(\mathbf{w}^* - \hat{\mathbf{w}}_t)^T (\hat{\mathbf{w}}_t - \mathbf{h}_t).
\end{align*}
In the first part of \eqref{eq_eq_1}, 
\begin{align*}
||\mathbf{w}^* - \mathbf{h}_t ||^2  = & || \mathbf{w}^* - \mathbf{w}_{t-1} + \mathbf{w}_{t-1} - \mathbf{h}_t ||^2 \\
 = & ||\mathbf{w}^* - \mathbf{w}_{t-1}||^2 + ||\mathbf{w}_{t-1} - \mathbf{h}_t||^2 + 2 (\mathbf{w}^* - \mathbf{w}_{t-1})^T (\mathbf{w}_{t-1} - \mathbf{h}_t) \\
= & ||\mathbf{w}^* - \mathbf{w}_{t-1}||^2 + ||\mathbf{w}_{t-1} - \eta \nabla_{\mathbf{w}_{t-1}} L(\mathbf{w}_{t-1}, z) - \mathbf{w}_{t-1} ||^2 \\
& + 2(\mathbf{w}^* - \mathbf{w}_{t-1})^T ( \mathbf{w}_{t-1} - \mathbf{w}_{t-1} + \eta \nabla_{\mathbf{w}_{t-1}} L(\mathbf{w}_{t-1}, z)) \\
= & ||\mathbf{w}^* - \mathbf{w}_{t-1}||^2 + \eta^2 || \nabla_{\mathbf{w}_{t-1}} L(\mathbf{w}_{t-1}, z)||^2 + 2 \eta (\mathbf{w}^* - \mathbf{w}_{t-1})^T\nabla_{\mathbf{w}_{t-1}} L(\mathbf{w}_{t-1}, z). \numberthis \label{eq_eq_2}
\end{align*}
Since $L(w, z)$ is convex, 
$$
(\mathbf{w}^* - \mathbf{w}_{t-1})^T \nabla_{\mathbf{w}_{t-1}} L(\mathbf{w}_{t-1}, z) \leq L(\mathbf{w}^* , z) - L(\mathbf{w}_{t-1}, z).
$$
And based on Assumption 1,
$$
||\nabla_{\mathbf{w}_{t-1}} L(\mathbf{w}_{t-1}, z)||^2 \leq A L(\mathbf{w}_{t-1}, z) + B.
$$
Thus \eqref{eq_eq_2} has the upper bound
$$
||\mathbf{w}^* - \mathbf{h}_t ||^2  \leq ||\mathbf{w}^* - \mathbf{w}_{t-1}||^2 + \eta^2 (AL(\mathbf{w}_{t-1}, z) + B) + 2\eta \left[  L(\mathbf{w}^* , z) - L(\mathbf{w}_{t-1}, z) \right].
$$

In the second part of \eqref{eq_eq_1}, since $\hat{\mathbf{w}}_t = T(\mathbf{h}_t, g_t)$, $\hat{\mathbf{w}}_t = \mbox{sign}(\mathbf{h}_t) \max ( |\mathbf{h}_t| -g_{0,t} \tilde{\mathbf{k}}_{t}, 0)$ and thus 
$$
 (\hat{\mathbf{w}}_t - \mathbf{h}_t)^T \hat{\mathbf{w}}_t  = -g_{0,t} \sum\limits_{j=1}^{p} \mbox{sign}(h_{t,j}) \tilde{k}_{t,j} \hat{\mathbf{w}}_{t,j}  = -g_{0,t} || \tilde{\mathbf{k}}_{t} \hat{\mathbf{w}}_t ||_1.
$$
Then
\vspace{-0.5cm}
\begin{align*}
- (\mathbf{w}^* - \hat{\mathbf{w}}_t ) ^T (\hat{\mathbf{w}}_t - \mathbf{h}_t) & = - (\mathbf{w}^*)^T (\hat{\mathbf{w}}_t - \mathbf{h}_t) + \hat{\mathbf{w}}_t^T (\hat{\mathbf{w}}_t - \mathbf{h}_t) \\
& \leq - \sum\limits_{j=1}^{p} |w_j^*| | \hat{w}_{t, j} - h_{t,j} |  - g_{0,t} ||\tilde{\mathbf{k}}_{t}\hat{\mathbf{w}}_t ||_1 \\
& \leq -g_{0,t} \sum\limits_{j=1}^{p} | \tilde{k}_{t,j} w_j^* | - g_{0,t}  || \tilde{\mathbf{k}}_{t} \hat{\mathbf{w}}_t ||_1 \\
& \leq -  K g_{0,t}  \sum\limits_{j=1}^{p} |w_j^* | - K g_{0,t}  ||\hat{\mathbf{w}}_t ||_1 \\
 & = -K g_{0,t} \left( || \mathbf{w}^* ||_1 - || \hat{\mathbf{w}}_t ||_1 \right).
\end{align*}
The second last inequality is due to $\tilde{k}_{j,t} \leq K$ for all $j = 1, \dots, p$. 

Thus, in total, equation \eqref{eq_eq_1} has the upper bound
\begin{align} \label{eq_eq_12}
||\hat{\mathbf{w}}_t - \mathbf{w}^* ||^2 \leq &  ||\mathbf{w}^* - \mathbf{w}_{t-1}||^2 + \eta^2 (AL(\mathbf{w}_{t-1}, z) + B) \\
& + 2\eta \left[  L(\mathbf{w}^* , z) - L(\mathbf{w}_{t-1}, z) \right] + 2K g_{0,t} \left( || \mathbf{w}^* ||_1 - || \hat{\mathbf{w}}_t ||_1 \right).
\end{align} 

On the other hand, from Lemma~\ref{lemma1} we have 
\begin{align*}
&\mathbb{E} \left( ||\hat{\mathbf{w}}_t - \mathbf{w}^*||^2 - ||\tilde{\mathbf{w}}_t - \mathbf{w}^* ||^2  \right)  \geq \varepsilon^2 ( |\hat{S}_{\varepsilon}| - |\hat{\Omega} | ) + 2 \pi_0 C^2  \left[  \left( 1 - \frac{ q^{\mathbf{g_t}}}{2\pi_0 p - p} \right) q^{\mathbf{g}_t} - d^* \right] \numberthis \label{eq_eq_4}, 
\end{align*}
where $\mathbf{g}_t = g_{0,t} \tilde{\mathbf{k}}_t $.

Since $\pi_0 \in \left( \frac{1}{2} + \frac{(q^{\mathbf{g}})^2}{2p(q^{\mathbf{g}} - d^*)} , 1\right)$ and, by Assumption 3, $q^{\mathbf{g}} > d^*$ for $t=1, \dots, T$, we have 
$$
 2 \pi_0 C^2  \left[  \left( 1 - \frac{ q^{\mathbf{g}_t}}{2\pi_0 p - p} \right) q^{\mathbf{g}_t} - d^* \right] \geq 0,  \quad \mbox{for } t = 1, \dots, T.
$$
Thus, we can further write \eqref{eq_eq_4} as

\begin{equation} \label{eq_eq_44}
\mathbb{E}  \left(||\tilde{\mathbf{w}}_t - \mathbf{w}^* ||^2  \right)  \leq \mathbb{E} \left( ||\hat{\mathbf{w}}_t - \mathbf{w}^*||^2 \right) - \varepsilon^2 ( |\hat{S}_{\varepsilon}| - |\hat{\Omega} | ).
\end{equation}
By rearranging \eqref{eq_eq_44} and combining it with \eqref{eq_eq_12}, we get
\begin{align*}
& \mathbb{E} \left( \left( 1 - \frac{1}{2} \eta A \right)  L (\mathbf{w}_{t-1}, z_t)  + \frac{K g_{0,t}}{\eta} ||\mathbf{w}_{t-1}||_1 \right) \\
 \leq &  \mathbb{E} \left( L(\mathbf{w}^*, z_t) + \frac{K g_{0,t}}{\eta} ||\mathbf{w}^*||_1 \right) + \frac{1}{2\eta} \mathbb{E} \left( ||\mathbf{w}^* - \mathbf{w}_{t-1}||^2 - ||\mathbf{w}^* - \tilde{\mathbf{w}}_t||^2 \right) + \frac{\eta B}{2} \\
 & - \frac{\varepsilon^2}{2\eta} ( | \hat{S}_{\varepsilon_t} | -| \hat{\Omega} |). \numberthis  \label{eq_eq_5}
\end{align*}

Based on Algorithm 3, let $\underline{\mathbf{w}}_t$ be the output weight vector at the end of time $t$ and let $\underline{g}_t$ be the applied gravity parameter at time $t$, where
\begin{align*}
\underline{g}_{0,t} & = \left\{ 
\begin{array}{ll}
\frac{g_{0,t}}{\eta}, & \mbox{if } t/K \mbox{ is an integer},\\
0, & \mbox{otherwise}.\\
\end{array}
\right. , \\
\underline{g}_{t} & = \left\{ 
\begin{array}{ll}
\tilde{\mathbf{k}} \frac{g_{0,t}}{\eta}, & \mbox{if } t/K \mbox{ is an integer},\\
0, & \mbox{otherwise}.\\
\end{array}
\right.,
\end{align*}
and,
$$
\underline{\mathbf{w}}_t = \left\{ 
\begin{array}{ll}
\tilde{\mathbf{w}}_t, & \mbox{if } t/(K n_K) \mbox{ is an integer},\\
\hat{\mathbf{w}}_t, & \mbox{otherwise}.\\
\end{array}
\right..
$$

By initializing the weight vector as a vector of zeros of length $p$, we sum up \eqref{eq_eq_5} over $t = 1, \dots, T$ with telescoping to obtain the following result:
\begin{align*}
& \left(1- \frac{1}{2} \eta A \right) \mathbb{E} \left( \sum\limits_{t=1}^{T} \left[  L(\underline{\mathbf{w}}_t, z_t) + \frac{K \underline{g}_{0,t} }{1-1/2 \eta A}  ||\underline{\mathbf{w}}_t||_1 \right] \right) \\
 \leq & \mathbb{E} \left( \sum\limits_{t=1}^{T} L(\mathbf{w}^*, z_t) + K \underline{g}_{0,t}  ||\mathbf{w}^*||_1 \right) + \frac{1}{2\eta} \mathbb{E} \left( || \mathbf{w}^* - \underline{\mathbf{w}}_1||^2 - || \mathbf{w}^* - \underline{\mathbf{w}}_T||^2 \right) + \frac{\eta BT}{2} \\
& - \frac{1}{2\eta} \sum\limits_{t=1}^{T} \varepsilon_t^2 \mathbbm{1} (\frac{t}{K n_K} \in \mathbb{Z} )  ( |\hat{S}_{\varepsilon, t}| - |\hat{\Omega}_t|) \\
\leq & \mathbb{E} \left( \sum\limits_{t=1}^{T} L(\mathbf{w}^*, z_t) + K \underline{g}_{0,t} ||\mathbf{w}^*||_1 \right)  +  \frac{|| \mathbf{w}^* ||^2 }{2\eta} + \frac{\eta BT}{2} \numberthis  \label{eq_eq_6} \\
& - \frac{1}{2\eta} \sum\limits_{t=1}^{T} \varepsilon_t^2 \mathbbm{1} (\frac{t}{K n_K} \in \mathbb{Z} )  ( |\hat{S}_{\varepsilon, t}| - |\hat{\Omega}_t|). 
\end{align*}

The theorem follows by rearrange terms in \eqref{eq_eq_6} and that $\tilde{k}_{t,j} \leq K$ for $j=1, \dots, p$ and $t=1, \dots, T$. 

\end{proof}

\end{appendices}

\end{document}